\documentclass[sigconf, nonacm]{acmart}

\usepackage{algorithmic}
\usepackage[ruled,vlined]{algorithm2e}
\usepackage{booktabs}
\usepackage{pifont}
\usepackage{fancyhdr} 
\usepackage{color}
\usepackage{soul, framed}
\usepackage{makecell}
\usepackage{xcolor}

\usepackage{colortbl}

\usepackage{graphicx}
\usepackage{textcomp}
\usepackage{xcolor,colortbl}
\usepackage{multirow}
\usepackage{bm,dsfont}
\usepackage{subfigure}
\usepackage{balance}
\usepackage{stmaryrd}
\usepackage{autobreak}
\usepackage{enumitem}
\usepackage{svg}
\usepackage{caption}
\usepackage{bbding}
\usepackage{wasysym}

\usepackage{textcomp}
\usepackage{fontawesome}
\usepackage{tcolorbox}

\newcommand{\netname}{{UnIMP}}
\newcounter{example}[section]
\renewcommand{\theexample}{\nthesection.\arabic{example}}

\newcounter{definition}[section]
\renewcommand{\thedefinition}{\nthesection.\arabic{definition}}
\newenvironment{definition}{
     \refstepcounter{definition}
     {\vspace{1ex} \noindent\bf  Definition  \thedefinition:}}{
     \vspace{1ex}} 

\newcounter{theorem}[section]
\renewcommand{\thetheorem}{\nthesection.\arabic{theorem}}
\newenvironment{theorem}{\begin{em}
        \refstepcounter{theorem}
        {\vspace{1ex} \noindent\bf  Theorem  \thetheorem:}}{
        \end{em}\vspace{1ex}} 

\newcounter{lemma}[section]
\renewcommand{\thelemma}{\nthesection.\arabic{lemma}}
\newenvironment{lemma}{\begin{em}
        \refstepcounter{lemma}
        {\vspace{1ex}\noindent\bf Lemma \thelemma:}}{
        \end{em}\vspace{1ex}} 

\newcounter{remark}[section]
\renewcommand{\theremark}{\nthesection.\arabic{remark}}

\newcommand{\nthesection}{\arabic{section}}

\newcommand{\beqn}{\begin{eqnarray*}}
\newcommand{\eeqn}{\end{eqnarray*}}

\newcounter{ccc}

\newcommand\vldbavailabilityurl{URL_TO_YOUR_ARTIFACTS}

\newcommand\vldbpagestyle{plain} 
\begin{document}

\title{On LLM-Enhanced Mixed-Type Data Imputation with High-Order Message Passing}

\author{Jianwei Wang$^{1}$, ~~ Kai Wang$^{2}$, ~~Ying Zhang$^{3}$,~~Wenjie Zhang$^{1}$, ~~Xiwei Xu$^{4}$,~~ Xuemin Lin$^{2}$}
\affiliation{\normalsize{$^{1}${The University of New South Wales, Sydney, Australia}}\\
 \normalsize{$^{2}${Antai College of Economics \& Management, Shanghai Jiao Tong University, Shanghai, China}} \\
 \normalsize{$^{3}${Zhejiang Gongshang University, Hangzhou, China}}\\
 \normalsize{$^{4}${Data61, CSIRO, Australia}}}
\email{{jianwei.wang1, wenjie.zhang}@unsw.edu.au, {w.kai, xuemin.lin}@sjtu.edu.cn, ying.zhang@zjgsu.edu.cn, xiwei.xu@data61.csiro.au}

\begin{abstract}

Missing data imputation, which aims to impute the missing values in the raw datasets to achieve the completeness of datasets, is crucial for modern data-driven models like large language models (LLMs) and has attracted increasing interest over the past decades. 
Despite its importance, existing solutions for missing data imputation either 1) only support numerical and categorical data or 2) show an unsatisfactory performance due to their design prioritizing text data and overlooking intrinsic characteristics of tabular data. 

In this paper, 
we propose \netname, a \textbf{\underline{Un}}ified \textbf{\underline{IMP}}utation framework that leverages LLM and high-order message passing to enhance the imputation of mixed-type data including numerical, categorical, and text data.
Specifically, we first introduce a cell-oriented hypergraph to model the table. We then propose BiHMP, an efficient \underline{Bi}directional \underline{H}igh-order \underline{M}essage-\underline{P}assing network to aggregate global-local and high-order information on the constructed hypergraph while capturing the inter-column heterogeneity and intra-column homogeneity.
To effectively and efficiently align the capacity of the LLM with the information aggregated by BiHMP, we introduce Xfusion, which, together with BiHMP, acts as adapters for the LLM. 
We follow a pre-training and fine-tuning pipeline to train \netname, integrating two optimizations: chunking technique, which divides tables into smaller chunks to enhance efficiency; and progressive masking technique, which gradually adapts the model to learn more complex data patterns. 
Both theoretical proofs and empirical experiments on 10 real-world datasets 
highlight the superiority of \netname~ over existing techniques.

\end{abstract}

 \maketitle
\pagestyle{\vldbpagestyle}

\ifdefempty{\vldbavailabilityurl}{}{
\begingroup\small\noindent\raggedright\textbf{PVLDB Artifact Availability:}\\
The source code, data, and/or other artifacts have been made available at \url{https://github.com/guaiyoui/UnIMP}.
\endgroup
}

\section{Introduction}

\label{sec:intro}


Data quality has gained increasing attention due to its pivotal role in developing and training data-driven models, particularly those built on artificial intelligence (AI) technology. For instance, the performance of large language models (LLMs) can be significantly enhanced by high-quality training datasets, without requiring substantial changes to the model architecture~\cite{dubey2024llama}. 
The missing data problem, as one of the most critical issues in the area of data quality, is ubiquitous in real-world raw datasets due to various factors such as poor data collection practices and device malfunctions~\cite{miao2022experimental}. To handle these issues, missing data imputation (a.k.a. missing value imputation)~\cite{wang2024missing,you2020handling} has been introduced, which aims to fill in missing values using information from observed data samples and features, thereby enhancing the quality of the raw dataset and improving the performance of downstream tasks. 

When applying imputation techniques to real-world datasets, it is desirable that these techniques can achieve high accuracy and support a variety of data types. High imputation accuracy is crucial, as existing studies have demonstrated that a higher accuracy would generally lead to a higher downstream task performance~\cite{wang2024missing, miao2022experimental, zhao2023transformed,yoon2018gain}. 
Supporting multiple data types is also important, as real-world datasets in fields like e-commerce~\cite{xie2019sentiment} and healthcare~\cite{johnson2016mimic} often contain a mix of numerical, categorical and text data~\cite{singh2023embeddings, yang2015entity}.
{\color{black}Moreover, with advances in the Internet of Things (IoT~\cite{madakam2015internet}) and big data, the proportion of mixed-type data continues to grow~\cite{wasti2022growing, mxied_data}.}

\begin{table}
\caption{Comparison of representative imputation methods. "---" denotes not applicable. Abbr: GAN (generative adversarial network), GNN (graph neural network), OT (optimal transport), AE (autoencoder), LLMs (large language models).}

 \vspace{-3mm}
 
 \centering \scalebox{0.75}{
\begin{tabular}
{@{}l@{\hspace{5pt}}c@{\hspace{8pt}}c@{\hspace{2pt}}c@{\hspace{5pt}}c@{\hspace{8pt}}c@{}}
\toprule[1.2pt]
 \multirow{2}{*}{\bf Methods}  & \multicolumn{3}{c}{\bf Handle Data Types} & \multirow{2}{*}{\bf \ Generalization} & \multirow{2}{*}{\bf Backbone}
    \\ \cline{2-4} 
    \multicolumn{1}{c}{} & {\bf Num.} & {\bf Cate.} & {\bf Text} & \multicolumn{1}{c}{}
\\ \midrule
Mean/Mode~\cite{jamshidian2007advances} & \ding{72}\ding{72}     & \ding{72}\ding{72} &    ---  &  \ding{72}\ding{72}\ding{72}  & Statistics         \\
KNNI~\cite{zhang2012nearest}   & \ding{72}\ding{72}\ding{72}      & \ding{72}\ding{72}\ding{72} & ---   &  \ding{72}\ding{72}\ding{72}      &  Similarity   \\
MICE~\cite{white2011multiple}   & \ding{72}\ding{72}\ding{72}     & \ding{72}\ding{72}\ding{72} & ---       & \ding{72}   & Regression      \\
GAIN~\cite{yoon2018gain}/VGAIN~\cite{miao2022experimental}    & \ding{72}\ding{72}\ding{72}    & \ding{72}\ding{72}\ding{72}  & ---  &  \ding{72}  & GAN \\
GRAPE~\cite{you2020handling}/IGRM~\cite{zhong2023data} & \ding{72}\ding{72}\ding{72}    & \ding{72}\ding{72}\ding{72} & --- & \ding{72} & GNNs  \\
TDM~\cite{zhao2023transformed}  & \ding{72}\ding{72}\ding{72}    & \ding{72}\ding{72}\ding{72} & --- & \ding{72}   & OT   \\
ReMasker~\cite{DBLP:conf/iclr/DuM024} & \ding{72}\ding{72}\ding{72}    & \ding{72}\ding{72}\ding{72} & --- & \ding{72} & AE  \\
NOMI~\cite{wang2024missing}  & \ding{72}\ding{72}\ding{72}    & \ding{72}\ding{72}\ding{72}  & ---      & \ding{72}    & Similarity      \\
DFMs~\cite{narayan2022can} & \ding{72}    & \ding{72}\ding{72}  & \ding{72}\ding{72}\ding{72} & \ding{72}\ding{72}\ding{72}  & LLMs    \\
Table-GPT~\cite{li2024table}/Jellyfish~\cite{zhang2024jellyfish} & \ding{72}    & \ding{72}\ding{72}  & \ding{72}\ding{72}\ding{72} & \ding{72}\ding{72}\ding{72}  & LLMs   \\
\midrule
\netname~(ours)  & \ding{72}\ding{72}\ding{72} & \ding{72}\ding{72}\ding{72} & \ding{72}\ding{72}\ding{72} & \ding{72}\ding{72}\ding{72} &  \makecell[c]{ GNNs+LLMs} \\ \bottomrule

\end{tabular}}

\vspace{-5mm}
\label{tab:intro_comparasion}
\end{table}

\begin{figure*}
  \centering
  \includegraphics[width=0.950\linewidth]{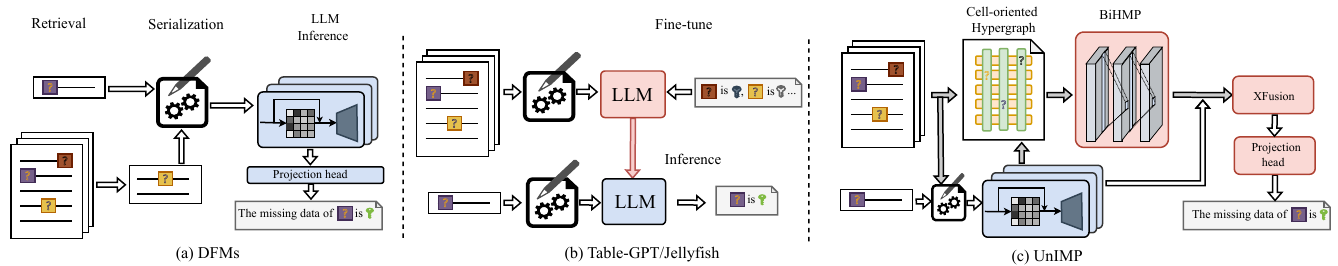}
  \vspace{-4mm}
  \caption{Comparisons of Frameworks. }
  \vspace{-5mm}
  \label{fig:frameworks_comparasion}
\end{figure*}

\noindent\textbf{Existing works.} Given the importance and practical value of missing data imputation, a set of techniques has been developed~\cite{wang2024missing, you2020handling, DBLP:conf/iclr/DuM024}. Some representative and recent methods are summarized in Table~\ref{tab:intro_comparasion}, and more details can be found in the surveys~\cite{miao2022experimental, lin2020missing}. 
The early-stage imputation methods were primarily rule-based, including statistics-based approaches such as Mean and Mode~\cite{jamshidian2007advances}, and similarity-based methods such as $K$ Nearest Neighbors Imputation (KNNI)~\cite{zhang2012nearest}. These methods have good generalization capacity, as the rules can be applied to various datasets. However, their imputation accuracy is often limited since these fixed rules often struggle to capture the underlying patterns in diverse datasets effectively.

To better model relationships among samples and features within datasets for imputation, various machine learning and deep learning methods are introduced, such as regression algorithm in MICE~\cite{white2011multiple}, generative adversarial net (GAN)~\cite{DBLP:journals/tkde/WuWMWY24} in GAIN~\cite{yoon2018gain} and VGAIN~\cite{miao2022experimental}, optimal transport (OT)~\cite{DBLP:journals/tkde/WuMNZHY24} in TDM~\cite{zhao2023transformed}, auto-encode (AE) in ReMasker~\cite{DBLP:conf/iclr/DuM024}, and graph neural networks (GNNs) in GRAPE~\cite{you2020handling} and IGRM~\cite{zhong2023data}.
While these methods achieve outstanding performance in handling numerical and categorical data, they struggle to process text data effectively. 
Furthermore, these approaches typically train a separate model for each dataset or even for each individual feature, hindering their generalizability to other datasets.

Recently, LLM-based methods are emerging in the field, leveraging knowledge encoded in LLM to predict missing data. LLM-based imputation can be broadly categorized into two primary categories: in-context-learning-based methods (e.g.,~\cite{narayan2022can, qian2024unidm, zhang2023large, biester2024llmclean}) and fine-tuning-based methods (e.g.,~\cite{li2024table, zhang2024jellyfish}).
The in-context-learning-based methods, with \underline{D}irect \underline{F}oundation \underline{M}odel\underline{s} (DFMs)~\cite{narayan2022can} as a prominent representative depicted in Figure~\ref{fig:frameworks_comparasion}(a), directly infer missing data by pre-trained LLM with carefully designed prompts.
The fine-tuning-based methods, with Table-GPT and Jellyfish as prominent representatives depicted in Figure~\ref{fig:frameworks_comparasion}(b), fine-tune LLM to adapt the model for tabular tasks including imputation.

\noindent\textbf{Motivations.} 
By serializing tabular data into text and generating the next token auto-regressively from prompts, existing LLM-based methods can handle mixed-type data with good generalization capacity.
However, they typically exhibit limited performance, especially with numerical data. They directly recast the problem of imputation as text generation while the intrinsic characteristics of tabular data remain insufficiently explored: 
1) \textit{Global-local information}. The value of a cell is influenced by both global patterns across the entire table and localized details of the cell. These LLM-based methods focus on individual samples and several examples as prompt, overlooking global information; 
2) \textit{High-order dependencies}. The relationships in the table are not necessarily pairwise but may involve three or more entities simultaneously. However, LLM-based methods, built on the attention mechanism~\cite{vaswani2017attention}, mainly focus on pairwise relationships. While stacking attention layers allows tokens later in the prompt to capture some high-order relationships, it lacks an explicit mechanism to capture them comprehensively; 
3) \textit{Inter-column heterogeneity and intra-column homogeneity}. Features across columns can be highly diverse, and cells within the same column generally exhibit consistent semantics. LLM-based methods, designed for sequential inputs, struggle to align the specific contextual relationships between columns and within columns.

As summarized in Table~\ref{tab:intro_comparasion}, existing imputation methods generally fall into two categories: 1) LLM-based methods exhibit an unsatisfactory performance, particularly for numerical data, and 2) other learning-based methods and rule-based methods only support numerical and categorical data. 
To better support real-world applications, an imputation method capable of handling mixed-type data while maintaining high accuracy is in demand.
A straightforward approach would be to impute each data type using the outperforming method for that type, such as applying MICE~\cite{royston2011multiple} or NOMI~\cite{wang2024missing} for numerical and categorical data, and Jellyfish~\cite{zhang2024jellyfish} for text data.
However, imputing each data type separately may ignore the intricate dependencies between them, leading to inconsistencies and reduced overall accuracy.


\noindent\textbf{Challenges}. To design an accurate imputation method for mixed-type data, two challenges exist below:

\textit{Challenge I: How to aggregate global-local and high-order information for imputation while capturing inter-column and intra-column patterns?} 
The existence of missing data and large-scale tables complicate the aggregation process.
A direct solution is to input the entire table into the LLM. However, computing complexity grows quadratically as the size of input increases.
Another promising alternative is to apply the global to local framework~\cite{wang2022global} to learn the multi-level features for imputation. However, it relies on similarity search on samples with missing features, leading to inaccurate results. Methods such as retrieval-augmented generation~\cite{lewis2020retrieval} show promise for aggregating information. However, they primarily focus on local information while overlooking global information. 
Therefore, it is challenging to aggregate these key information.

\textit{Challenge II: How to effectively integrate the aggregation module and train the model to support mixed-type imputation?}
The presence of mixed-type data complicates the integration process. 
To handle mixed-type data, a direct approach would be to sequentially apply LLM, aggregation module and then another LLM.
This pipeline first uses the LLM to embed the input into a shared latent space, followed by information aggregation and finally generates the output by another LLM.
However, treating LLM and the aggregation module as separate components may limit the alignment of the capabilities encoded in LLM with the aggregated information. 
Moreover, training this sequential pipeline on tables of varying sizes may fail to adequately utilize computational resources, resulting in inefficiencies.
Hence, it is challenging to integrate the aggregation module and train the model for mixed-type imputation.

\noindent\textbf{Our solutions}. Guided by the above challenges, we propose \netname, a \textbf{\underline{Un}}ified \textbf{\underline{IMP}}utation framework as depicted in Figure~\ref{fig:frameworks_comparasion}(c) that employs LLM and the high-order message passing for accurate mixed-type data imputation. 

To address \textit{Challenge I}, we design a cell-oriented hypergraph in \netname~to model tabular datasets and introduce BiHMP, an efficient and effective \underline{Bi}directional \underline{H}igh-order \underline{M}essage-\underline{P}assing network, to aggregate information on the constructed hypergraph. 
By leveraging high-order message passing on the hypergraph, BiHMP learns to aggregate global-local and high-order information while capturing intra- and inter-column relationships.
Most existing hypergraphs for tables are value-oriented, modeling distinct values as nodes~\cite{du2022learning}. However, this approach struggles with modeling unseen values in the context of missing data and handling numerical and textual data where co-occurrence patterns are not easily defined, limiting its applicability in mixed-type imputation.
Therefore, we adopt the cell-oriented hypergraph for imputation following~\cite{chen2024hytrel}, where each cell in the dataset is modeled as a node and nodes within the same row or column are connected by hyperedges. 
This structure enables the flexible integration of missing and mixed-type data.
Furthermore, BiHMP is designed to efficient yet effective aggregate information.
It comprises iteratively applied linear hyperedge-to-node and node-to-hyperedge layers. The hyperedge-to-node layer enriches node representations by aggregating information from their associated hyperedges. Conversely, the node-to-hyperedge layer refines hyperedge features by propagating information from connected nodes back to their hyperedges.

To address \textit{Challenge II}, \netname~is designed with the LLM as the backbone, while BiHMP and Xfusion (a fusion module based on attention mechanism~\cite{vaswani2017attention}) serve as adapters~\cite{zhang2023llama, he2021effectiveness} for the LLM.
Specifically, prompts are together propagated by both the LLM and the BiHMP on the embedded hypergraph. The information aggregated by BiHMP, along with the LLM-generated data, is sent to the Xfusion module to enhance the alignment and integration of both sets of information. The fused data is subsequently sent to the projection head of LLM to generate output.
Moreover, The adapter strategy trains only the BiHMP, Xfusion and projection head, freezing most LLM parameters to enhance efficiency and better preserve pre-trained knowledge in LLM while incorporating BiHMP-aggregated information.

To more effectively and efficiently train \netname~for mixed-type imputation, we follow the pre-train and fine-tune strategy~\cite{alexandr2021fine,li2023label,DBLP:journals/pvldb/WangWLZZ24} and integrate two optimizations, i.e., the chunking and the progressive masking motivated by~\cite{yepes2024financial, you2020handling,DBLP:conf/iclr/DuM024}.
Specifically, the chunking technique splits all tables into smaller, uniform chunks to handle large tables and facilitate efficient batch processing. This approach allows for more effective use of computational resources, such as GPU memory and parallel computation capabilities.
The progressive masking technique, on the other hand, gradually increases the complexity of the imputation task by progressively masking more data during training. 
This technique enables the model to learn complex patterns progressively, improving its understanding of data relationships and generalization capacity.

\noindent\textbf{Theoretical and empirical studies}. We theoretically and empirically demonstrate the superiority of \netname. Theoretically, we prove that a model capable of capturing global-local information, high-order dependencies, and inter- and intra-column patterns leads to improved imputation accuracy. Furthermore, we show that our \netname~effectively incorporates these critical properties.
Empirical evaluations on 10 real-world datasets with varying table sizes and data types also validate the excellent performance of \netname~in terms of both accuracy and efficiency. 
For numerical and categorical data, \netname~ reduces imputation RMSE by an average of 12.11\% to 45.22\% compared to previous state-of-the-art methods and other baselines. In the case of imputing text data, \netname~ improves imputation accuracy by 30.53\%, 28.04\% and 23.47\% over DFMs~\cite{narayan2022can}, Table-GPT~\cite{li2024table} and Jellyfish~\cite{zhang2024jellyfish} respectively, as measured by  $\text{ROUGE-1}_{F1}$. 
Furthermore, \netname~ can handle large datasets and exhibits robust generalization ability.

\noindent \textbf{Contributions.}
The main contributions are as follows.
\begin{itemize}[leftmargin=10pt, topsep=0pt]

\item{} We propose a unified framework \netname~ to accurately impute mixed-type data, including numerical, categorial and text data.


\item{} We propose a cell-oriented hypergraph to model tabular data and design BiHMP, an efficient and effective bidirectional high-order message-passing network, to aggregate information while capturing the key intrinsic characteristics of tabular data.

\item{} We introduce Xfusion, along with BiHMP, as adapters for LLM to process diverse data and enable accurate mix-type imputation. Additionally, we integrate chunking and progressive masking techniques to train \netname, enhancing performance.

\item{} We theoretically prove the critical role of global-local information, high-order relationships, and inter- and intra-column patterns to the accuracy of imputation

\item{} Extensive empirical experiments on 10 real-world datasets validate the outstanding performance of \netname.

\end{itemize}

\vspace{-2mm}
\section{Preliminaries}
\label{sec:pre}

\begin{table}[t]
\centering 
\caption{Symbols and Descriptions}
\vspace{-0.4cm}
\label{tab:symbol}
\begin{tabular}{|p{2.0cm}|p{6.2cm}|}
\hline
\cellcolor{lightgray}\textbf{Notation} & \cellcolor{lightgray}\textbf{Description} \\ \hline\hline
$X, \Tilde{X}, \overline{X}$& raw dataset, imputed dataset and ground-truth \\ \hline
$n, d$& the row number and column number \\ \hline
$x_i, x_{ij}$ & $i$-th sample in $X$ and its $j$-th column data \\ \hline
$M \in \{0, 1\}^{n\times d}$ & the mask matrix indicating incompleteness\\ \hline
$m_i, m_{ij}$ & $i$-th sample in $M$ and its $j$-th column data \\ \hline
$HG(\mathcal{V}, \mathcal{E})$ & hypergraph with nodes $\mathcal{V}$ and hyperedges $\mathcal{E}$ \\ \hline
$\theta$ & model parameters. \\ \hline
$z_{v_i}, z_{e_i} $ & the hidden embedding of node $v_i$ \& hyperedge $e_i$ \\ \hline
$\sigma $ & non-linear function (ReLU) \\ \hline
$f, g $ & functions or layers in the neural networks\\ \hline

\end{tabular}
\vspace{-5mm}
\end{table}
\vspace{-1mm}
\subsection{Problem Statement}
\vspace{-1mm}

Following previous works~\cite{wang2024missing, miao2022experimental}, the task of missing data imputation is defined over tabular datasets $\mathcal{X}$~\cite{li2024table, miao2022experimental, zhao2023transformed, wang2024missing}, where each $X \in \mathcal{X}$ consists of $n$ data samples (rows) and $d$ features (columns). We denote the $j$-th feature of the $i$-th sample as $x_{ij}$, which can be of numerical, categorical or text type.
A mask matrix ${M} \in {R}^{n\times d}$ is used to indicate missing components in $X$.
$m_{ij}$ equals 1 if $x_{ij}$ is observed and 0 if $x_{ij}$ is missing, i.e., $X=X_{obs}\cup X_{miss}$ where $X_{obs}=\{x_{ij}|x_{ij}\in X, m_{ij}=1\}$ and $X_{miss}=\{x_{ij}|x_{ij}\in X, m_{ij}=0\}$. 
The imputed data and the ground-truth data are denoted as $\Tilde{X}$ and $\overline{X} = {X}_{obs} \cup \overline{X}_{miss}$, respectively. 
Note that the indexing in this paper starts from zero, which is a common convention in many programming languages.

The hypergraph is a generalization of the graph where each hyperedge may contain an arbitrary number of nodes. We denote a hypergraph by $HG(\mathcal{V}, \mathcal{E})$, where $\mathcal{V}$ is a set of nodes, and $\mathcal{E}$ is the set of hyperedges. $\mathcal{E} \subseteq P^*(\mathcal{V})\setminus\{\varnothing\}$ where $P^*(\mathcal{V})$ is the power set on the nodes~\cite{chen2024hytrel}. We use $x_v$ to denote the cell w.r.t. the node $v$.
With the above notations, we define the task of mixed-type missing data imputation as follows. 

\begin{definition}{(Mixed-type Missing Data Imputation~\cite{wang2024missing, you2020handling}).} 
\textit{Mixed-type missing data imputation} aims to impute the unobserved elements in the raw data, i.e., $X_{miss}$, and make the imputed matrix $\Tilde{X}$ as close to the real complete dataset $\overline{X}$ as possible. The raw data matrix ${X}$ may contain numerical, categorical and text data.
\end{definition}

\vspace{-4mm}
\subsection{Large Language Models}

LLMs, such as GPT~\cite{brown2020language} and Llama~\cite{dubey2024llama} which contain billions of parameters, have demonstrated remarkable emergent behaviors and impressive zero-shot generalization capabilities across diverse tasks. The execution process of LLMs typically comprises four key components: 1) a tokenizer that segments input text into discrete tokens, 2) an LLM backbone that is used for feature propagation, 3) a projection head that predicts the next token, and 4) a decoding process that translates token IDs back into human-readable text. 

These LLMs are often auto-regressive, trained on large text corpora with the objective of maximizing the log-likelihood of the next word given the previous words. 

\begin{equation}
\label{equ:autoregressive}
\begin{aligned}
   \theta_{LLM} = \mathop{\arg\max}\limits_{\theta}\sum\limits_{i} \log P(t_i|t_{i-k},\cdots,t_{i-1};\theta)
\end{aligned}
\end{equation}
where $t_i$ stands the token and $k$ is the context window size.
There are two typical ways to adapt the model to new tasks and boost performance: supervised-fine-tuning (SFT) and in-context-learning.
For SFT, given a training query $T = \{t_0, t_1, \cdots, t_{s}\}$ and training target $y$, the objective is to maximize the following log-likelihood:
\(\sum_{(T, y)} \log P(y|t_{0}, t_{1},\cdots,t_{s})\).
The in-context-learning-based methods handle new tasks during inference by constructing a prompt that includes examples (i.e., context). Given a set of example queries $\{T^0, T^1, \cdots, T^{w}\}$, targets $\{y^0, y^1, \cdots, y^{w}\}$, and a new query text $T^{w+1}$, a prompt $(T^0, y^0, \cdots, T^{w}, y^{w}, T^{w+1})$ is constructed and sent to the LLM to infer the results.

Methods employing both techniques have been adapted for missing data imputation. Specifically, Table-GPT~\cite{li2024table} and Jellyfish~\cite{zhang2024jellyfish} use SFT to fine-tune the language model on diverse tabular tasks, showing impressive performance gains compared to the original LLM. DFMs~\cite{narayan2022can} leverage the in-context-learning by constructing prompts with $k$-shots context (i.e., $k$ query-target pairs) for imputation, demonstrating the impressive performance of LLM.

However, these LLM-based methods are inherently designed to prioritize text data and cannot effectively utilize key intrinsic characteristics of tabular data, leading to unsatisfactory performance, especially when handling numerical data.

\section{The Overall Framework}
\label{sec:method}

In this section, we first introduce three key intrinsic characteristics of tabular data and prove its critical role for imputation. Then, we introduce the details of \netname.

\subsection{Intrinsic Characteristics of Tabular Data}
\label{sec:inductive_basees}

\noindent\textbf{Global-local information}.
The global-local information indicates that the value of the cell is influenced by both its neighboring cells (local context) and cells across the table (global context).
For example, as shown in Figure~\ref{fig:frameworks_overview}, the name of the president is relevant not only to their nation and term but also to the sequential relationship of terms. Based on this, we can infer that the missing value corresponds to the president succeeding President Trump.

\noindent\textbf{High-order relationship}.
The high-order relationship suggests that the value of the cell is influenced simultaneously by multiple columns as a whole. For example, as shown in Figure~\ref{fig:frameworks_overview}, the name of a president can be determined simultaneously by the nation and the term as a whole. However, neither the nation nor the term alone is sufficient to fully determine the value.

\noindent\textbf{Intra-column heterogeneity and intra-column homogeneity}. 
Features across columns can be highly diverse, and cells within the same column typically exhibit consistent semantics. For example, as shown in Figure~\ref{fig:frameworks_overview}, the name is text data and the nation is categorical data. Furthermore, the name remains consistent across rows.

\begin{figure*}
  \centering
  \includegraphics[width=0.90\linewidth]{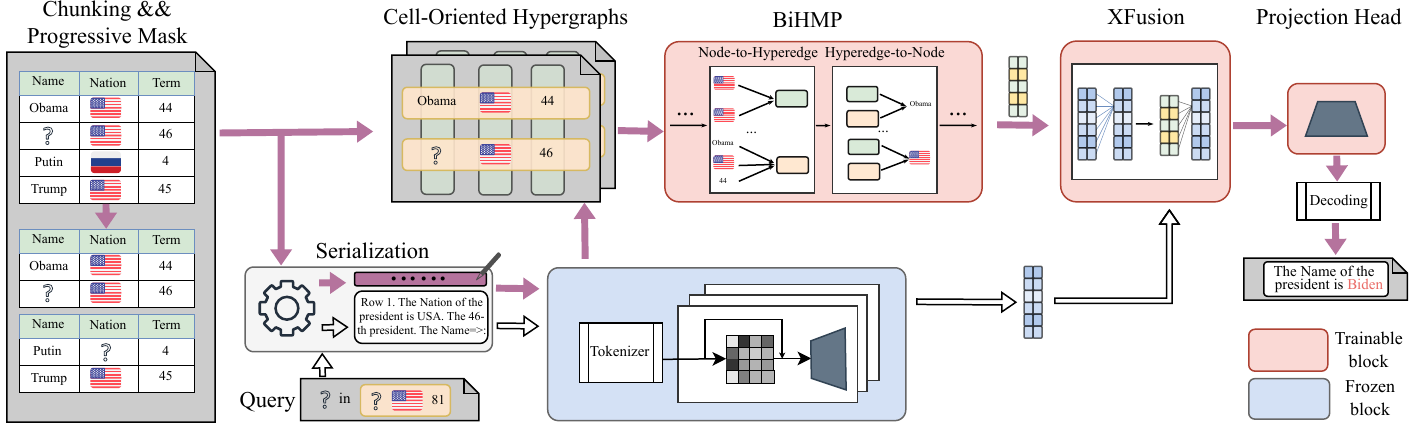}
  \vspace{-2mm}
  \caption{Framework overview of \netname }
  \vspace{-4mm}
  \label{fig:frameworks_overview}
\end{figure*}

Next, we formally define the imputation error.

\begin{definition}{(Imputation Error).}
    Given a model ${\theta_{X_{obs}}}$ trained on the observed data $X_{obs}$, the imputation error $\Psi({\theta_{X_{obs}}}, X_{miss})$ captures the expected performance of ${\theta_{X_{obs}}}$ on the test unobserved data $X_{miss}$, i.e., $\Psi({\theta_{X_{obs}}}, X_{miss})=\mathbb{E}[\|{\theta_{X_{obs}}}(X_{miss}) - \overline{X}_{miss}\|]$. When the context is clear, we denote $\theta_{X_{obs}}$ as $\theta$ to improve readability.
\end{definition}

We then have the following three theorems highlighting the importance of global-local information, high-order information and column patterns, respectively. The proofs are in Section~\ref{sec:analysis}.

\begin{theorem}
\label{theorem:global_local}
    Consider two imputation models, ${\theta^{g+l}}$ and ${\theta^{l}}$, where ${\theta^{g+l}}$ captures both global and local information in the latent space, and ${\theta^{l}}$ captures only the local information. Assuming that interactions of global and local information are independent, then we have:  
    \[
    \Psi({\theta^{g+l}}, X_{miss}) \leq \Psi({\theta^{l}}, X_{miss}),
    \]
    indicating that a model capable of capturing both global and local information achieves a lower imputation error.

\end{theorem}

\begin{theorem}
\label{theorem:high_order}
    Consider two imputation models, ${\theta^{[0:r]}}$ and ${\theta^{[0:s]}}$, where ${\theta^{[0:r]}}$ captures interactions up to order $r$ in the latent space, and ${\theta^{[0:s]}}$ captures interactions up to order $s$, with $r > s$. 
    We have  
    \[
    \Psi({\theta^{[0:r]}}, X_{miss}) \leq \Psi({\theta^{[0:s]}}, X_{miss}),
    \]
    indicating that the model capable of capturing higher-order interactions exhibits a lower imputation error.
\end{theorem}

\begin{theorem}
\label{theorem:column_pattern}
Consider two imputation models, \(\theta^{cp}\) and \(\theta\), where \(\theta^{cp}\) captures the column patterns including intra-column heterogeneity and intra-column homogeneity, while \(\theta\) does not. Then, we have:  
\[
\Psi(\theta^{cp}, X_{\text{miss}}) \leq \Psi(\theta, X_{\text{miss}}),
\]  
indicating that model \(\theta^{cp}\) capturing the column patterns achieves a lower imputation error.
\end{theorem}

\vspace{-3mm}
\subsection{Overview of \netname}
Guided by the above key intrinsic characteristics of tabular data, we propose \netname~for accurate mixed-type data imputation.
The overall architecture is illustrated in Figure~\ref{fig:frameworks_overview}.
Given the incomplete raw dataset, \netname~ first utilizes the chunking technique and progressive masking to process data. Then, cell-oriented hypergraphs are constructed. 
\netname~utilizes the tokenizer and LLM backbone for feature initialization and propagation, projecting both raw data and query prompts from the text space to the embedding space. 
These obtained features are processed through the BiHMP module, which iteratively runs the node-to-hyperedge layer and the hyperedge-to-node layer to aggregate both local and global information. The aggregated feature is then fused with the embedded prompts using an XFusion module. 
Finally, a projection head maps the predicted data back to token space, and a decoding process is employed to translate the tokens into human-readable results.

\vspace{-2mm}

\subsection{Cell-Oriented Hypergraph Modeling}
\label{sec:hypergraph_modeling}
\noindent\textbf{Motivation}. 
To capture the structural properties of tabular data, existing imputation methods often employ graph-based approaches, such as bipartite graphs~\cite{zhong2023data, you2020handling} and similarity graphs~\cite{chen2023gedi}. While effective, these methods fall short in comprehensively representing the complex interactions and higher-order relationships inherent in tabular data. 
To better capture the intrinsic characteristics of tabular data, we adopt a hypergraph-based approach for modeling. 
While some hypergraph-based methods, such as the cell-oriented hypergraph in~\cite{chen2024hytrel} and the value-oriented method in~\cite{du2022learning}, have been introduced for general tabular representation learning, they have not yet been employed for imputation.
The value-oriented approach models each distinct value as a node, with nodes in the same row sharing a hyperedge. However, this method struggles to deal with mixed-type data and effectively handles missingness. 
Moreover, missing data imputation focuses on the granularity of individual cells, as missingness occurs at the cell level. 
Therefore, we adopt the cell-oriented hypergraph following ~\cite{chen2024hytrel}.
It is important to note that the work in~\cite{chen2024hytrel} focuses on tabular representation learning, our work is specifically designed for missing data imputation. Additionally, their method treats hypergraph neural network and language models as separate components, whereas we introduce BiHMP and Xfusion, which serve as the adapters for LLM, to better aggregate information and combine both advantages.

Given a tabular dataset $X$ with $n$ samples, each containing $d$ features, we construct a hypergraph $HG(\mathcal{V}, \mathcal{E})$ as follows:
\begin{itemize}[leftmargin=10pt, topsep=1pt]
\item For each cell $x_{ij} \in X$, we create a corresponding node $v_{idx} \in \mathcal{V}$, where $idx = i*d+j$.
\item For nodes in the same column (i.e., nodes corresponding to $\{x_{0j}, x_{1j}, \cdots\}$), we construct a hyperedge $e_{j} \in \mathcal{E}$.
\item Similarly, nodes in the same row (i.e., nodes corresponding to $\{x_{i0}, x_{i1}, \cdots\}$) form a hyperedge $e_{i+d} \in \mathcal{E}$.
\end{itemize}

\vspace{-2mm}
\subsection{Feature Encoding}
As the neural network models need vectorized inputs, we introduce a vectorization technique to project data from the text space into the vector space. We use the LLM backbone to encode the text data.

\noindent\textbf{Serialization}. 
In order to obtain the embedding, we need to feed the proper prompt text to the LLM backbone.
We use the following serialization prompt for node data. Here, EOS (i.e., end of the sentence) is a special token to denote the end of the prompt.
\[
  \text{Row i, col\_name=>\{node data\} EOS }
\]
To provide additional context for the prompt when making a query, we also concatenate the other observed data in the same row along with their column names to the prompt.
We use the following serialization prompt for hyperedge data.
\[
  \text{This is row (or col): i (or col\_name) EOS}
\]

\noindent\textbf{Tokenization}. 
Then, the text is sent into the tokenizer to split the prompt into tokens. There are many commonly used tokenizers like BPE~\cite{berglund2023formalizing} and SentencePiece~\cite{kudo2018sentencepiece}. In our experiments, we use the SentencePiece tokenizer associated with Llama2.

\begin{equation}
\tag*{}
\label{equ:tokenizer}
\begin{aligned}
\{t_0, t_1, \cdots, t_s\} = \text{tokenizer}(\text{prompt-text})
\end{aligned}
\end{equation}

\noindent\textbf{Propagation of LLM backbone}.
The tokens are processed through the LLM backbone for feature initialization and propagation. Initially, the LLM backbone performs a lookup in the embedding matrix, a predefined matrix with dimensions $|\text{Vocab}| \times d$, where $|\text{Vocab}|$ denotes the vocabulary size and $d$ represents the hidden dimension. Each token ID corresponds to a specific row in this embedding matrix, serving as an index to retrieve the corresponding embedding vector.
These embedding vectors are then fed into the transformer layer \cite{vaswani2017attention} for feature propagation.

\begin{equation}
\tag*{}
\label{equ:tokenizer}
\begin{aligned}
z_p:\{z_{t_0}, z_{t_1}, \cdots, z_{t_s}\} = \text{LLM-backbone}({t_0, t_1, \cdots, t_s})
\end{aligned}
\end{equation}
Each token $t_i$ is encoded as a feature representation $z_{t_i}$. 
For the prompt text corresponding to node $v_i$ (or hyperedge $e_i$), we use the representation of the last token as the initial feature $z^0_{v_i}$ (or $z^0_{e_i}$, respectively).
For numerical data, we directly encode the input as its value to avoid unnecessary computational overhead. For categorical features, we employ label encoding, which assigns unique integers to each category, aligning with recent works \cite{hancock2020survey, zhao2023transformed, wang2024missing, wang2022global, miao2022experimental}.

\subsection{Bidirectional High-order Message Passing}

\noindent\textbf{Motivation}. 
In tabular data, interactions often involve multiple entities simultaneously. For instance, multiple cells within the same row and column exhibit high-order interactions. To capture these complexities, Hypergraph Neural Networks (HGNNs) (e.g., \cite{feng2019hypergraph}) can be employed for message passing, but they typically come with high computational costs. While a hypergraph-structure-aware transformer was proposed in~\cite{chen2024hytrel}, it not only faces efficiency challenges but also tends to lose fine-grained local information due to its maximal invariant properties.
A simplified and fast HGNN~\cite{Tang2024SimplifyingHN} can accelerate training, but it only handles node information, overlooking hyperedges.
Therefore, we propose BiHMP, a novel Bidirectional High-order Message Passing network that operates bidirectionally between nodes and hyperedges in the constructed hypergraph. BiHMP comprises two key components: 1) Node-to-Hyperedge layer, which propagates information from each node to its associated hyperedges through a linear layer.
2) Hyperedge-to-Node layer, which propagates information from hyperedges back to their constituent nodes through a linear layer. Through the iterative application of these two layers, BiHMP captures high-order and multi-hop relationships within the data, enhancing its ability to model complex interactions.

\noindent\textbf{Node-to-Hyperedge}. For each hyperedge $e_j \in \mathcal{E}$ and node $v_i \in e_j$, with their respective embeddings $z^{l}_{e_j}$ and $z^{l}_{v_i}$ in $l$-th Node-to-Hyperedge layer, we use the following equation to learn to update the representation of hyperedge:

\begin{equation}
\label{equ:v2e}
\begin{aligned}
    z_{e_j}^{temp} &= \frac{1}{|e_j|}\sum_{v_i \in e_j}\sigma \left( f_1^l(z^{l}_{v_i})\right) \\
    z_{e_j}^{l+1} &= \sigma \left( f_2^l\left(\text{CONCAT}\left( z^{l}_{e_j}, z_{e_j}^{temp}\right)\right)\right)
\end{aligned}
\end{equation}
where $f_1^l(\cdot)$ and $f_2^l(\cdot)$ represents learnable linear transformation functions. 
$\sigma(\cdot)$ is a non-linear activation function (ReLU in \netname). $\text{CONCAT}(\cdot)$ is the concatenation function that stacks multiple inputs into a single, longer vector. For each node in the hyperedge, we apply a learnable transformation followed by a non-linear activation $\phi(\cdot)$. We then compute the mean of these transformed node features to obtain $z_{e_i}^{temp}$, which captures the average information from all nodes in the hyperedge. This temporary representation is concatenated with the representation of the hyperedge from the previous layer, $z_{e_j}^l$. Finally, we apply another transformation layer and non-linear activation layer to this concatenated vector to obtain the updated hyperedge representation $z_{e_j}^{l+1}$.

\noindent\textbf{Hyperedge-to-Node}. The Hyperedge-to-Node layer updates node representations by aggregating information from their incident hyperedges. For each node $v_i \in \mathcal{V}$ with its two incident hyperedges i.e., the row-level hyperedge $e_{v_i}^r$ and column-level hyperedge $e_{v_i}^c$, we update the node representation by the following equation in the $l$-th Hyperedge-to-Node layer:
\begin{equation}
\label{equ:e2v}
z_{v_i}^{l+1} = \sigma\left(f^l_3\left(\text{CONCAT}\left[z^{l}_{v_i}, z^{l}_{e_{v_i}^c}, z^{l}_{e_{v_i}^r}\right]\right)\right)
\end{equation}
where $f_3^l(\cdot)$ is a learnable linear function. The Hyperedge-to-Node layer first concatenates the embedding of the node $z^l_{v_i}$ with the embeddings of all hyperedges that contain the node, i.e., $z^{l}_{e_{v_i}^c}$ and $z^{l}_{e_{v_i}^r}$. 
The concatenated representation is then processed through a learnable linear transformation function $f_3^l(\cdot)$.
Finally, a non-linear activation function $\sigma(\cdot)$ is applied to obtain the updated node representation in $(l+1)$-th layer $z_{v_i}^{l+1}$. Note that, the main component in our BiHMP is the linear transformation which is efficient yet effective for information propagation.

\begin{algorithm}[t]
\caption{Forward Propagation of \netname}
\label{algo:forwardpropagation}
\LinesNumbered
\DontPrintSemicolon
\KwIn{Raw dataset $X$, Mask $M$, model parameters $\theta_{\text{\netname}}$. }
\KwOut{The imputed dataset $\tilde{X}$.}

\tcp*[l]{Hypergraph Construction and Feature Encoding.}

$HG(\mathcal{V}$, $\mathcal{E})$ $\leftarrow$ Construct hypergraph from 
$X$; $\tilde{X}\leftarrow X$\;

\For{$v_i \in \mathcal{V}$, $e_j \in \mathcal{E}$}{
    $\{t_0, \cdots, t_s\}$ $\leftarrow$ Serilization and tokenization. \;
    $z_p:\{z_{t_0}, z_{t_1}, \cdots, z_{t_s}\} \leftarrow \text{LLM-backbone}({t_0, t_1, \cdots, t_s})$ \;
    $z^0_{v_i}\ \text{or}\ z^0_{e_j} \leftarrow z_{t_s}$\;
}

\tcp*[l]{High-order Message Passing.}
\For{$l=0,\cdots, l_{max}$}{
    $z_{e_j}^{temp} = \frac{1}{|e_j|}\sum_{v_i \in e_j}\sigma \left( f_1^l(z^{l}_{v_i})\right)$
    \
    
    $z_{e_j}^{l+1} = \sigma \left( f_2^l\left(\text{CONCAT}\left( z^{l}_{e_j}, z_{e_j}^{temp}\right)\right)\right)$\
    
    $z_{v_i}^{l+1} = \sigma\left(f^l_3\left(\text{CONCAT}\left[z^{l}_{v_i}, z^{l}_{e_{v_i}^c}, z^{l}_{e_{v_i}^r}\right]\right)\right)$
}

\tcp*[l]{Xfusion and projection head.}
\For{$v_i$ with corresponding mask being 0}{
    $\tilde{x}_{v_i}= \varnothing$ \;
    \While{True}{
    Update $z_p$ with $\tilde{x}_{v_i}$\;
    $z_p^{temp} = \text{Attn}(z_p W_Q, z_p W_K, z_p W_V; w_1)$ \;
    $z_{output} = \text{Attn}(z_p^{temp} W_Q, z_g W_K, z_g W_V; w_2)$\;
    $o = f^{head}(z_{output})$[-1]; $\tilde{x}_{v_i}.append(o)$\;
    \If{Decode(o) = EOS or $x_{v_i}$ is num. or cate.}{
    Break\;
}
}       
}
\Return $\tilde{X}$\;
\end{algorithm}

\subsection{XFusion Block and Projection Head}

\noindent\textbf{XFusion block}. After encoding the prompt and aggregating local and global information by BiHMP, it is crucial to effectively integrate these features for accurate imputation. We designed XFusion to synthesize the rich information derived from the prompt feature (denoted as $z_p=\{z_{t_0}, z_{t_1}, \cdots, z_{t_s}\} $) with the features extracted from the hypergraph representation (denoted as $Z_g$ where $z_g \in Z_g$ is a concatenation of $z^{l_{max}}_{e_{v_i}^c}$ and $z^{l_{max}}_{e_{v_i}^r}$ for imputing node $v_i$). XFusion is based on the attention mechanism~\cite{vaswani2017attention}, a critical component of LLMs. The attention is computed as follows:

\begin{equation}
\text{Attn}(Q, K, V; w) = w \cdot \text{Softmax}\left(\frac{QK^T}{\sqrt{d_K}}\right)V
\end{equation}
where $Q$, $K$ and $V$ are the input feature matrices, $w$ is a learnable weight matrix, and $d_K$ is the dimensionality of the vectors in $K$.

Building on this attention mechanism, our XFusion module leverages both token embeddings and hypergraph-derived embeddings to construct a rich, context-aware representation. The forward process of XFusion is as follows:

\begin{equation}
\label{equ:fusion}
\begin{aligned}
z_p^{temp} &= \text{Attn}(z_p W_Q, z_p W_K, z_p W_V; w_1) \\
z_{output} &= \text{Attn}(z_p^{temp} W_Q, z_g W_K, z_g W_V; w_2)
\end{aligned}
\end{equation}

Here, $W_Q$, $W_K$ and $W_V$ are learnable weight matrices. We first calculate the self-attention of the prompt embedding. We then use the obtained prompt feature $h_p^{temp}$ as the query and the graph feature $h_g$ as both the key and value matrices. The attention mechanism between prompt embeddings and graph embeddings allows the model to effectively align and integrate the prompt information with the local and global information from the graph embeddings.

\noindent\textbf{Projection head}. The obtained $z_{output} \in \mathbb{R}^{(s+1)\times d_h}$ is in a high-dimensional space, and the projection head is a key component that maps this high-dimensional embedding to a target space. It is modeled as a linear layer $f^{head}: z_{output}  \mapsto o \in \mathbb{R}^{(s+1)\times d_o}$. For numerical and categorical data, $d_o=1$, as it requires only a single value as the imputation. In contrast, for text data, the task involves predicting a token ID, making it a classification task. Hence, $d_o = |\text{Vocab}|$, where $|\text{Vocab}|$ denotes the size of the vocabulary.

With the above introduction of key components in \netname, we now show that the \netname~captures the key intrinsic characteristics analyzed in Section~\ref{sec:inductive_basees}. The proof is in Section~\ref{sec:analysis}.

\begin{theorem}
\label{theorem:all_in_one}
    The \netname~framework captures the global-local information, high-order relationship, inter-column heterogeneity and intra-column homogeneity for imputation. 
\end{theorem}

The overall forward process of \netname~is summarized in Algorithm~\ref{algo:forwardpropagation}. \netname~takes the raw dataset $X$ and mask $M$ as inputs and outputs the imputed dataset $\tilde{X}$. We first construct a hypergraph according to the process introduced in Section~\ref{sec:hypergraph_modeling} (Line 1). 
It then initializes the feature of each node and hyperedge (Lines 2 to 5). The initialized features are then propagated by the high-order message-passing framework (Lines 6 to 9). It contains $l_{max}$ layers and each layer contains both the Node-to-Hyperedge layer (Lines 7 and 8) and Hyperedge-to-Node (Line 9). 
The aggregated feature and the prompt representation are then fused by the Xfusion for projection (Lines 10 to 18). Prompt embedding is first updated by previous predicted $\tilde{x}_{v_i}$ (Line 13). The Xfusion is applied to the obtained features (Lines 13 to 14). The projection head is applied to predict the imputation, using the last token as the newly predicted token (Line 16). If the predicted token is EOS or its type is numerical or categorical, it terminates the auto-regressive process (Lines 17 to 18).
At last, the algorithm outputs the imputed matrix.

\vspace{-2mm}
\section{Training and Analysis}
\label{sec:analysis}

\subsection{Training Objectives}
\noindent\textbf{Motivation}. The training objective guides the model towards desired behaviors and outputs during the learning process.
For the text data, it may contain an arbitrary length of tokens. Therefore, it is often modeled as the classification of the token ID of the next token and is processed in an auto-regressive manner. For the numerical data and the label-encoded categorical data, the imputation is often modeled as the task of regression. 

The learning process for text data is handled using an auto-regressive approach. 
For a sequence of tokens $(t_0, t_1, \cdots, t_s)$, an autoregressive model tries to predict the next element in the sequence, given all the previous tokens. This objective is based on Cross Entropy (CE)~\cite{de2005tutorial, wang2024neural} as follows:
\begin{equation}
\label{equ:loss_text}
\begin{aligned}
    & \mathcal{L}(T;\theta) = P(t_0) \cdot P(t_1|t_0) \cdots P(t_s|t_0, t_1, \cdots, t_{s-1})\\
    &= \prod_{i=0}^s \textrm{CE}(t_i, \theta(t_0\cdots t_{i-1})) 
    = \prod_{i=0}^s \sum_{j=0}^{d_o}-t_{i}[j]\cdot\textrm{log}(\theta(t_0\cdots t_{i-1})[j])
\end{aligned}
\end{equation}
Where $t_{i}[j]$ and $\theta(\cdot)[j]$ are the probabilities of corresponding values belonging to $j$-th class.

The learning objective for numerical and categorical data is formulated using the Huber Loss~\cite{huber1992robust}. For a true value $x_{ij}$ and predicted value $\Tilde{x}_{ij}$, Huber Loss is defined:
\begin{equation}
\label{equ:loss_huber}
\mathcal{L}(x_{ij}, \Tilde{x}_{ij}) = \begin{cases}
    \frac{1}{2}(x_{ij} - \Tilde{x}_{ij})^2 & \text{for } |x_{ij} - \Tilde{x}_{ij}| \leq \delta \\
    \delta|x_{ij} - \Tilde{x}_{ij}| - \frac{1}{2}\delta^2 & \text{otherwise}
\end{cases}
\end{equation}
where $\delta$ is a positive hyper-parameter. Compared to other loss functions, Huber Loss offers the following advantages: 1) For errors smaller than $\delta$, Huber Loss behaves like Mean-Squared-Error (MSE), being sensitive to small errors; 2) For errors larger than $\delta$, it behaves like Mean-Absolute-Error (MAE), being less sensitive to potential outliers. Moreover, Huber Loss is differentiable at all points, which is beneficial for optimization. By adjusting the value of $\delta$, it is flexible to balance between MSE and MAE, adapting to different contexts. We set $\delta=1$ suggested by the experiment in Section~\ref{sec:experiment}.

\subsection{Training Pipeline}

\noindent\textbf{Motivation}. Previous works in imputation~\cite{you2020handling,zhao2023transformed,yoon2018gain,wang2024missing} typically train a separate model for each dataset. While effective, this practice significantly increases computational costs, as training a model can be time-intensive. Additionally, models trained on individual datasets often struggle to generalize to other datasets. A common alternative in the field of LLM is the "pre-train and finetune" strategy. However, the variability in table sizes and the propensity of neural models to overfit present challenges.
In this work, we follow the "pre-train and finetune" paradigm for imputation, introducing two plug-and-play optimizations: the chunking technique and the progressive masking technique. 

\noindent\textbf{Optimization 1: Chunking technique}. We split the entire dataset into smaller, uniformly sized chunks (512 rows in each chunk in our experiments). This approach ensures that all chunks are processed efficiently and in parallel, enabling batch training and reducing the heavy IO time associated with handling large and irregularly sized datasets. By transforming the dataset into manageable chunks, we also facilitate more effective memory utilization, preventing bottlenecks during model training.

\noindent\textbf{Optimization 2: Progressive masking technique}. Starting with the raw dataset, we gradually increase the proportion of masked data. The newly masked ratio is set from $\kappa$ to $\kappa + 30\%$ where $\kappa$ is hyperparameter and is set to 35\% suggested by the experiments in Section~\ref{sec:experiment}. As training progresses, we mask more data, gradually exposing the model to more challenging imputation scenarios. This progressive masking strategy allows the model to incrementally learn patterns from simpler to more complex missing data scenarios and thus improve the generalization ability.

We summarize the overall pre-training pipeline in Algorithm~\ref{algo:training_procedure}. Following the SFT, the fine-tuning process has a similar procedure to the pre-training. The pre-training stage utilizes all available datasets, while fine-tuning focuses on specific datasets.
The pipeline inputs training datasets, chunk size, batch size, initial model parameters, learning rate and the number of epochs. It first splits the training datasets into batched chunks (Line 1) and initializes the optimizer with the given learning rate (Line 2). The model is trained for $\textit{Epochs}$ epochs (Lines 3–12).
During each epoch, we mask more data, and for each batch, the algorithm performs forward propagation (as described in Algorithm~\ref{algo:forwardpropagation}) to compute the imputation results (Line 6). For numerical or categorical data, the loss is calculated using Eqn~\ref{equ:loss_huber}, while for text data, the loss is computed using Eqn~\ref{equ:loss_text}.
After processing all cells in a chunk, the model parameters are updated using the optimizer and the loss (Line 12). Finally, the learned parameters are returned (Line 13).

\begin{algorithm}[t]
\captionsetup{singlelinecheck=false} 
\captionsetup{margin={0pt,5em}}
\caption{\parbox{\linewidth}{Pre-Training Pipeline}}
\label{algo:training_procedure}
\LinesNumbered
\DontPrintSemicolon
\KwIn{The tabular datasets $\mathcal{T}$ with its masks $\mathcal{M}$, chunk size $c_{size}$, batchsize $b_{size}$, model parameters $\theta_{\text{\netname}}$, learning rate $\varphi$, Epoch num $Epochs$, mask rate $\kappa$. }

$\{X^c, M^c\}\leftarrow$Split $\mathcal{X}, \mathcal{M}$ into batched chunks by $c_{size}$ \& $b_{size}$\;
Initialize optimizer $opt_{\theta}$ with learning rate $\varphi$\;

\For{$Epoch \in \{0, 1,\cdots, Epochs\}$}{
    $^{\prime}M^c \leftarrow $ mask $(\kappa+\frac{Epcoh}{Epochs}\times 0.30)$ of $M^c$ \;
    \For{each $X^c \in \{X^c\}$}{
    $\Tilde{X}^c  \leftarrow$ Algorithm~\ref{algo:forwardpropagation}($X^c,\  ^{\prime}M^c,\ \theta_{\text{\netname}}$)\;
    \For{each $x_{ij} \in T^c$ and $M^c_{ij} = 0$}{
    \If{$x_{ij}$ being numerical or categorical}{
         $\mathcal{L} \leftarrow $ Compute loss by Eqn~\ref{equ:loss_huber}\;
    }
    \Else{
        $\mathcal{L} \leftarrow $ Compute loss by Eqn~\ref{equ:loss_text}\;
    }
     }
     Update $\theta_{\text{\netname}}$ by $opt{\theta}$ with loss $\frac{\mathcal{L}}{\left|X_c\right|}$.\;
    }
}
Return $\theta_{\text{\netname}}$

\end{algorithm}

\vspace{-2mm}
\subsection{Complexity Analysis}

The feature encoding process requires encoding $|\mathcal{V}|+|\mathcal{E}|$ prompts, with the time complexity of LLM inference being $O(l_{trans}\times t_{max}^2\times d_{LLM})$~\cite{zhou2024survey}, where $l_{trans}$ is the number of transformer layers, $t_{max}$ is maximum number of tokens, and $d_{LLM}$ is the dimension of embeddings in LLM. Therefore, the total time complexity of feature encoding is $O((|\mathcal{V}|+|\mathcal{E}|)\times l_{trans}\times t_{max}^2\times d_{LLM})$. 

The training pipeline consists of $L$ layers of Node-to-Hyperedge and Hyperedge-to-Node operations with complexities $O(d_{LLM}\times |\mathcal{V}|+d_{LLM}\times |\mathcal{E}|)$ and $O(d_{LLM}\times |\mathcal{E}|)$, respectively. An attention-based fusion block requires $O(t_{max}^2\times d_{LLM})$, and a linear projection head needs $O(d_{LLM})$ complexity. Accordingly, the overall time complexity of the training pipeline is $O(Epochs\times t_{max}\times(l_{max}\times(2\times d_{LLM}\times |\mathcal{V}|+d_{LLM}\times |\mathcal{E}|)+t_{max}^2\times d_{LLM}+d_{LLM}))$.

\vspace{-2mm}
\subsection{Theoretical Proofs}
In this part, we present the proofs for Theorem~\ref{theorem:global_local}, Theorem~\ref{theorem:high_order}, Theorem~\ref{theorem:column_pattern} and Theorem~\ref{theorem:all_in_one}. We first introduce the concepts of entropy and mutual information which are the base of the proof.

\begin{definition}
    $H(X)=- \sum_{x \in {X}} p(x) \log p(x)$ represents the entropy that quantifies the average level of information associated with $X$.
    $I(X;Y)=H(X)-H(X|Y)$ is the mutual information between the $X$ and $Y$. 
\end{definition}

\begin{lemma}
\label{lemma:entropy_relation}
    $H(X|Y) \geq H(X|Y, Z))$ and the equation get if the information of $Z$ is encoded in $Y$.
\end{lemma}
\begin{proof}
    \(H(X, Z \mid Y) = H(Z \mid Y) + H(X \mid Y, Z)\). Therefore $H(X \mid Y, Z) = H(X, Z \mid Y) - H(Z \mid Y)$. As $H(Z \mid Y)\geq 0$, we can obtain the inequality.
\end{proof}

\noindent\textbf{Proof of Theorem~\ref{theorem:global_local}}. The imputation error reflects how well the model trained on the training dataset generalizes to the unobserved data.
According to the information bottleneck theory~\cite{kawaguchi2023does, enwiki:1260066467}, $\Psi({\theta}_{X_{obs}}, X_{miss})$ is proven to scale as \(\tilde{O}(\sqrt{\frac{I(X_{miss}; Z_{miss})+1}{n_{obs}}})\) where and $n_{obs}$ is the number of training observed data.
For model ${\theta^{g+l}}$, its mutual information is 
\(I^{g+l}(X_{miss}; Z_{miss})=I(X_{miss}; Z_{miss}|B^{g+l})=I(X_{miss}; Z_{miss}|B^{g}, B^{l})\) where $B^{g+l}$ is the captured global and local information.
For model ${\theta^{l}}$, its mutual information is \(I^{l}(X_{miss}; Z_{miss})=I(X_{miss}; Z_{miss}|B^{l})\). Then we have 
\begin{align*}
    & I(X_{miss}; Z_{miss}|B^{l}) = I(X_{miss}; Z_{miss},B^{l})-I(X_{miss};B^{l})  \\
    &= H(X_{miss})-H(X_{miss}|Z_{miss},B^{l})-H(X_{miss})+H(X_{miss}|B^{l})\\
    &= H(X_{miss}|B^{l})-H(X_{miss}|Z_{miss},B^{l}) \\
    &\approx H(X_{miss}|B^{l})-H(X_{miss}|Z_{miss},B^{l},B^{g}) \\
    &\geq  H(X_{miss}|B^{l}, B^{g})-H(X_{miss}|Z_{miss},B^{l},B^{g})\\
    &= I(X_{miss}; Z_{miss}|B^{l}, B^{g})
\end{align*}
The first line is obtained by the definition of conditional mutual information. The second line is obtained by the definition of mutual information. The fourth line and fifth line are based on Lemma~\ref{lemma:entropy_relation}. The fourth line is obtained as $Z_{miss}$ contains the information in $B^{g}$ and add this term will not increase the entropy. The sixth line can be obtained by the reversal of the first three lines. 
Therefore, we can have $\Psi({\theta^{g+l}}, X_{miss}) \leq \Psi({\theta^{l}}, X_{miss})$.


\noindent\textbf{Proofs of Theorems~\ref{theorem:high_order} and~\ref{theorem:column_pattern}}.
The proofs of Theorem~\ref{theorem:high_order} and Theorem~\ref{theorem:column_pattern} are similar to that of Theorem~\ref{theorem:global_local}. 
Given ${\theta^{[0:r]}}$, its mutual information can be represented as $I^{[0:r]}(X_{miss}; Z_{miss})=I(X_{miss}; Z_{miss}|B^{[0:r]}) = I(X_{miss}; Z_{miss}|B^{[0:s]}, B^{[s+1:r]})$. Here, we use $B^{[0:r]}$ to represent the captured interactions from order 0 to order r. 
For model ${\theta^{[0:s]}}$, its mutual information is 
\(I^{g+l}(X_{miss}; Z_{miss})=I(X_{miss}; Z_{miss}|B^{[0:s]})\) where $B^{[0:s]}$ is the captured interactions from order 0 to order s and $r\geq s$. 
Then, we can get 
\( \Psi({\theta^{[0:r]}}, X_{miss}) \leq \Psi({\theta^{[0:s]}}, X_{miss})\) by replacing $B^{g+l}$ and $B^{l}$ in the proof of Theorem~\ref{theorem:global_local} with $B^{[0:r]}$ and $B^{[0:s]}$, respectively.
Similarly, given a model ${\theta^{cp}}$, its mutual information can be calculated as $I^{cp}(X_{miss}; Z_{miss})=I(X_{miss}; Z_{miss}|B^{cp}, B^{base})$ where $B^{cp}$ is the captured column patterns beside the base pattern $B^{base}$. 
Then, we can get the inequality
\( \Psi({\theta^{cp}}, X_{miss}) \leq \Psi({\theta}, X_{miss})\) by replacing $B^{g+l}$ and $B^{l}$ in the proof of Theorem~\ref{theorem:global_local} with $B^{cp}$ and $B^{base}$, respectively.

\noindent\textbf{Proof of Theorem~\ref{theorem:all_in_one}}.
\textit{1): Global-local information}. Any cell \(x_{ij}\) can reach the selected cell \(x_{rs}\) in two hops, as \(x_{ij}\) and \(x_{is}\) are connected by a hyperedge (they are in the same row), and \(x_{is}\) and \(x_{rs}\) are connected by a hyperedge (they are in the same column). 
\textit{2): High-order information}. The hyperedges enable the capture of features from multiple cells within the same row and column, allowing \netname~to model high-order relationships effectively.
\textit{3): Inter-column heterogeneity and intra-column homogeneity}. Each cell is modeled as a distinct node in the hypergraph which can better capture the inter-column heterogeneity. Moreover, the usage of column hyperedge can better capture the intra-column homogeneity.

\vspace{-1mm}
\section{Experimental Evaluation}
\label{sec:experiment}

\begin{table}
\centering
\caption{The profiles of datasets}
\vspace{-4mm}
\label{tab:dataset}

\begin{tabular}{|c|c|m{1.4cm}|m{0.9cm}|m{0.9cm}|m{0.9cm}|}
\hline
\cellcolor{lightgray}{Dataset} & \cellcolor{lightgray}{Alias} & \multicolumn{1}{c|}{\cellcolor{lightgray}\textbf{$n$}}  & \multicolumn{1}{c|}{\cellcolor{lightgray}$d_\text{num}$} & \multicolumn{1}{c|}{\cellcolor{lightgray}$d_\text{cate}$} & \multicolumn{1}{c|}{\cellcolor{lightgray}$d_\text{text}$} \\ \hline
\hline

{Blogger}& BG & \centering 100   & \centering 0 & \centering 6 & \multicolumn{1}{>{\centering\arraybackslash}m{1.0cm}|}{0} \\ \hline
{ZOO}& ZO & \centering 101   & \centering 0 & \centering 17 & \multicolumn{1}{>{\centering\arraybackslash}m{1.0cm}|}{0} \\ \hline
{Parkinsons}& PK & \centering 195   & \centering 22 & \centering 1 & \multicolumn{1}{>{\centering\arraybackslash}m{1.0cm}|}{0} \\ \hline
{Bike}& BK & \centering 8,760   & \centering 12  & \centering 1 & \multicolumn{1}{c|}{0} \\ \hline
{Chess} & CS & \centering 28,055   & \centering 3 & \centering 4 & \multicolumn{1}{c|}{0} \\ \hline
{Shuttle} & ST & \centering 43,500   & \centering 0  & \centering 10 & \multicolumn{1}{c|}{0} \\ \hline
{Power}  & PW  & \centering 2,049,280   & \centering 6  & \centering 0 & \multicolumn{1}{c|}{0}\\ \hline
{Buy}& BY & \centering 651   & \centering 1   & \centering 1 & \multicolumn{1}{c|}{2} \\ \hline
{Restaurant}  & RR &  \centering 864   & \centering 0   & \centering 2 & \multicolumn{1}{c|}{3} \\ \hline
{Walmart}  & WM & \centering 4,654   & \centering 0   & \centering 3 & \multicolumn{1}{c|}{2} \\ \hline
\end{tabular}

\end{table}
\vspace{-1mm}
\subsection{Experiment Setup}

\noindent\textbf{Datasets description}. 
The evaluation uses 10 real-world datasets from UCI~\cite{Dua:2019} and Kaggle~\cite{Kaggle:2019}, following the previous works~\cite{miao2022experimental, wang2024missing, DBLP:conf/iclr/DuM024, mei2021capturing}. The profiles of these datasets are presented in Table~\ref{tab:dataset}. These datasets exhibit diverse characteristics, including varying numbers of samples ($n$), numerical features ($d_{\text{num}}$), categorical features ($d_{\text{cate}}$) and text features ($d_{\text{text}}$).

\noindent\textbf{Baseline methods}. We include 13 baselines for a comprehensive evaluation. Specifically, we incorporate 10 baselines for numerical and categorical data, including: 1) MEAN~\cite{jamshidian2007advances}, which imputes missing values using the feature average; 2) KNNI~\cite{zhang2012nearest}, which imputes based on the weighted sum of the $k$ nearest neighbors; 3) MICE~\cite{white2011multiple}, which iteratively trains regressors for imputation; 4) VGAIN~\cite{miao2022experimental}, which integrates a Variational AE (VAE) with GAIN~\cite{yoon2018gain} for imputation; 5) GINN~\cite{spinelli2020missing}, which constructs a similarity graph and uses a graph auto-encoder for imputation; 6) GRAPE~\cite{you2020handling}, which builds a bipartite graph and applies link prediction for imputation; 7) IRGM~\cite{zhong2023data}, which combines similarity and bipartite graphs for imputation; 8) TDM~\cite{zhao2023transformed}, which matches samples from the latent space for imputation; 9) NOMI~\cite{wang2024missing}, which augments samples and uses a neural network Gaussian process for imputation; and 10) ReMasker~\cite{DBLP:conf/iclr/DuM024}, which reconstructs instances for imputation.

Additionally, we compare our method to three LLM-based approaches that handle multi-type data imputation: 1) DFMs~\cite{narayan2022can} which utilizes the in-context learning of LLM to infer imputation, 
2) Table-GPT~\cite{li2024table}, a fine-tuned LLM designed for various tabular tasks, including imputation,
and 3) Jellyfish~\cite{zhang2024jellyfish}, which finetunes LLMs for data preprocessing.

For our model, we include two variants, i.e., \netname~which is a generalized model trained on all datasets using one unified parameter set and \netname-ft which is fine-tuned on specific datasets.

\begin{table*}[thb] \centering
    \caption{Results of missing data imputation (20\% MCAR)}
    \vspace{-3mm}
    \label{tab:MCAR}
    \resizebox{\textwidth}{!}{
    \Huge\setlength{\tabcolsep}{5mm}
        \begin{tabular}{c*{4}{c}*{5}{c}*{5}{c}*{4}{c}}
        \toprule
        & \multicolumn{8}{c}{\textbf{RMSE}} 
        & \multicolumn{8}{c}{\textbf{MAE}} \\
        \cmidrule(lr){2-9}\cmidrule(lr){10-17}
                               Model & {Blogger} & {Zoo} & {Parkinsons}  
                                & {Bike} & {Chess} & {Shuttle} &  {Power} &  {\cellcolor[RGB]{255,200,200}{Improve\%}}
                               & {Blogger} & {Zoo} & {Parkinsons}  
                                & {Bike} & {Chess} & {Shuttle} &  {Power} &  {\cellcolor[RGB]{255,200,200}{Improve\%}} \\
        \midrule \midrule
MEAN & 0.4315 & 0.4260 & 0.2062 & 0.2239 & 0.3079 & 0.0914 & 0.0869 & 45.22\%              & 0.3631 & 0.3663 & 0.1473 & 0.1561 & 0.2584 & 0.0467 & 0.0559 & 57.24\%              \\
KNNI & 0.4417 & 0.2749 & 0.1891 & 0.2023 & 0.3246 & 0.0456 & OOT    & 35.75\%              & 0.3337 & 0.1473 & 0.1207 & 0.1277 & 0.2606 & 0.0192 & OOT    & 41.33\%              \\
MICE & 0.4134 & 0.2645 & 0.1263 & 0.1796 & 0.2966 & 0.0426 & 0.0650 & 28.27\%              & 0.3605 & 0.1731 & 0.0667 & 0.1097 & 0.2453 & 0.0131 & 0.0370 & 36.49\%              \\
VGAIN & 0.4316 & 0.4114 & 0.1913 & 0.2219 & 0.2797 & 0.0786 & 0.0811 & 42.48\%              & 0.3643 & 0.1891 & 0.1156 & 0.1363 & 0.2537 & 0.0352 & 0.0509 & 48.88\%              \\
TDM & 0.4384 & 0.2949 & 0.1862 & 0.2444 & 0.3027 & 0.0769 & 0.0926 & 41.48\%              & 0.3229 & 0.1490 & 0.0830 & 0.1499 & 0.2345 & 0.0350 & 0.0600 & 42.96\%              \\
GINN & 0.4657 & 0.2761 & 0.1466 & 0.1641 & 0.2911 & 0.0734 & OOM    & 32.41\%              & 0.3444 & 0.1445 & 0.0884 & 0.0921 & 0.2339 & 0.0434 & OOM    & 36.97\%              \\
GRAPE & 0.4304 & 0.3211 & 0.1064 & 0.1481 & 0.2749 & 0.0242 & OOM    & 20.89\%              & 0.3151 & 0.1605 & 0.0535 & 0.0796 & 0.2200 & 0.0073 & OOM    & 21.39\%              \\
IGRM & 0.4551 & 0.3063 & \textcolor{blue}{0.1035} & OOM    & OOM    & OOM    & OOM    & 12.11\%              & 0.3423 & 0.1621 & \textcolor{blue}{0.0497} & OOM    & OOM    & OOM    & OOM    & 8.74\%               \\
DFMs & 0.4413 & 0.4445 & 0.2412 & 0.2483 & OOT    & OOT    & OOT    & 41.53\%              & 0.3676 & 0.2934 & 0.1647 & 0.1529 & OOT    & OOT    & OOT    & 49.81\%              \\
Table-GPT & 0.4237 & 0.4315 & 0.2246 & 0.2547 & OOT    & OOT    & OOT    & 39.71\%              & 0.3572 & 0.2713 & 0.1761 & 0.1442 & OOT    & OOT    & OOT    & 48.99\%              \\
Jellyfish & 0.4133 & 0.4177 & 0.2127 & 0.1935 & OOT    & OOT    & OOT    & 36.43\%              & 0.3557 & 0.2719 & 0.1548 & 0.1478 & OOT    & OOT    & OOT    & 47.38\%              \\
NOMI & 0.4112 & \textcolor{blue}{0.2576} & 0.1322 & 0.1582 & 0.3042 & \textcolor{blue}{0.0237} & 0.0731 & 28.32\%              & \textcolor{blue}{0.3102} & \textcolor{blue}{0.1442} &  0.0710 & 0.0740 & 0.2298 & \textcolor{blue}{0.0071} & 0.0463 & 30.86\%              \\
ReMasker & \textcolor{blue}{0.4068} & 0.3309 & 0.1508 & \textcolor{blue}{0.1277} & 0.2662 & 0.1111 & OOT    & 29.61\%              & 0.3293 & 0.1807 & 0.0995 & \textcolor{blue}{0.0655} & 0.2113 & 0.0545 & OOT    & 35.36\%              \\
\midrule
\netname & 0.4171 & 0.2979 & 0.1407 & 0.1730 & \textcolor{blue}{0.2628} & 0.0398 & \textcolor{blue}{0.0485} & 25.84\%              & 0.3384 & 0.1822 & 0.0966 & 0.1121 & \textcolor{blue}{0.2050} & 0.0238 & \textcolor{blue}{0.0256} & 36.08\%              \\
\netname-ft& \textcolor{red}{0.3972} & \textcolor{red}{0.2474} & \textcolor{red}{0.0990} & \textcolor{red}{0.1172} & \textcolor{red}{0.2142} & \textcolor{red}{0.0134} & \textcolor{red}{0.0425} & \multicolumn{1}{c}{---} & \textcolor{red}{0.3082} & \textcolor{red}{0.1428} & \textcolor{red}{0.0475} & \textcolor{red}{0.0602} & \textcolor{red}{0.1438} & \textcolor{red}{0.0040} & \textcolor{red}{0.0225} & \multicolumn{1}{c}{---}\\
\bottomrule
    \end{tabular}
 }
    \vspace{0.01em}
    \begin{flushleft}
        \footnotesize $*$ \textcolor{red}{Red} text indicates the best result. \textcolor{blue}{Blue} text indicates the second best result. 'OOT' indicates out of time (with a limit of 10 hours). 'OOM' indicates out of memory.
    \end{flushleft}
    \vspace{-3mm}
\end{table*}

\noindent\textbf{Metrics}.
For numerical and categorical data, we adopt the Root-Mean-Square-Error (RMSE) and Mean-Absolute-Error (MAE) metrics, following previous works ~\cite{miao2022experimental, wang2022global, yoon2018gain}, to assess accuracy.

For text data, we employ two metrics to assess imputation accuracy: Recall-Oriented Understudy for Gisting Evaluation (ROUGE-1)~\cite{DBLP:conf/conll/NallapatiZSGX16} and Cosine Similarity (Cos-Sim)~\cite{DBLP:conf/emnlp/ReimersG19}, covering both lexical-wise and semantic-wise approaches. ROUGE-1 offers a statistical evaluation by comparing the overlap of unigrams between the imputed and original text, providing a straightforward measure of how closely the imputed text aligns with the original in terms of vocabulary usage. However, some imputations may convey the same meaning but differ in format or wording, such as "METRO" and "SUBWAY." To account for this, we introduce Cosine Similarity, which evaluates the semantic similarity by comparing the vector representations of the imputed and original texts. This metric captures deeper contextual relationships even when their lexical-level representation differs. Given the generated text $T_G$ and the ground-truth reference text $T_R$, the metrics are computed as:
\vspace{-2mm}
\begin{align*}
\text{ROUGE-1}_{\text{recall}} = \frac{| T_R \cap T_G |}{| T_G |}; \text{ROUGE-1}_{\text{precision}} = \frac{| T_R \cap T_G |}{| T_R |} \\[8pt]
\text{ROUGE-1}_{F_1} = 2 \cdot \frac{\text{ROUGE-1}_{\text{recall}} \cdot \text{ROUGE-1}_{\text{precision}}}{\text{ROUGE-1}_{\text{recall}} + \text{ROUGE-1}_{\text{precision}}}
\end{align*}
\vspace{-2mm}
\begin{equation}
    \text{Cos-Sim} = \frac{\sum_{i=1}^{d_{Bert}} BERT(T_G)_i \cdot BERT(T_R)_i}{\sqrt{\sum_{i=1}^{d_{Bert}} BERT(T_G)_i^2} \sqrt{\sum_{i=1}^{d_{Bert}} BERT(T_R)_i^2}}
\end{equation}
Where $BERT(\cdot)$ represents the embedding vectors given input. $d_{Bert}$ is the dimension of the embeddings. 

For the RMSE and MAE, a lower value indicates a better imputation. For the $\text{ROUGE-1}_{F_1} $ and Cos-Sim, a higher value indicates a better imputation.
It is important to note that for the first seven datasets (BG, ZO, PK, BK, CS, ST and PW) which contain only numerical and categorical features, we assess imputation performance using RMSE and MAE. On the other hand, since text data can only be evaluated with $\text{ROUGE-1}_{F_1}$ and Cos-Sim, we use these metrics to uniformly evaluate the last three datasets (BY, RR, WM), which include a mix of numerical, categorical and text features.

\noindent\textbf{Missing mechanisms}. Original data are fully observed, we thus follow previous works~\cite{miao2022experimental, zhao2023transformed} to generate the mask matrix. Three mechanisms are utilized, i.e., missing completely at random ({MCAR}), missing at random (MAR) and missing not at random (MNAR).

\noindent\textbf{Implementation details}. The missing rate is set as 20\% and the missing mechanism is MCAR by default. For a fair comparison, the pre-trained LLM is llama2-7b~\cite{touvron2023llama} (more LLMs are evaluated in Exp-8) for all the baselines. We use Jellyfish-7B to keep the same size of LLM in~\cite{zhang2024jellyfish}. As the model in ~\cite{li2024table} is not released, we fine-tune llama2-7b using the provided datasets. The number of layers is 3 in BiHMP. For the numerical and categorical data (\textit{resp}. text data), we set the chunk size as 512 (\textit{resp}. 32) with a batch size of 64 (\textit{resp}. 2). We trained for 4000 (\textit{resp}. 100) epochs and fine-tuned with 1000 (\textit{resp}. 20) epochs. 
We pre-train and infer the model in the server with Xeon(R) Silver 4314 CPU, 504GB memory and 2*A5000 (GPU).

\begin{table}[t] \centering
    \caption{Results of imputation over text data}
    \vspace{-3mm}
    \label{tab:string_result}
    \resizebox{0.45\textwidth}{!}{
    \Huge\setlength{\tabcolsep}{4mm}
        \begin{tabular}{c*{3}{c}*{3}{c}}
        \toprule
        & \multicolumn{3}{c}{\textbf{$\text{ROUGE-1}_{F_1}$}} 
        & \multicolumn{3}{c}{\textbf{Cos-Sim}} \\
        \cmidrule(lr){2-4}\cmidrule(lr){5-7}
                               Model & {Buy} & {Restaurant} & {Walmart} & {Buy} & {Restaurant} & {Walmart} \\
        \midrule \midrule
DFMs     & 0.1535	& 0.0822	& 0.1420 & 0.8251 &	0.7609 &	0.7943 \\
Table-GPT  &  0.1784 &	0.1398 &	0.1344 & 0.8345 &	0.8137 &	0.8254     \\
Jellfish   & 0.2153 &	0.1675 &	0.2067 & 0.8418 &	0.8145 &	0.778 \\

\netname  & 0.3327 &	0.4017 &	0.5594 & 0.8610 &	0.8774 &	0.9025 \\
\netname-ft & 0.4273 &	0.4326 &	0.5931 & 0.8892 &	0.8923 &	0.9177 \\
        \bottomrule
    \end{tabular}
    
 }
 \vspace{-6mm}
\end{table}
\vspace{-2mm}
\subsection{Accuracy Comparison}

\begin{figure*}
\subfigbottomskip=-8pt 
\subfigcapskip=-7pt 
    
    \subfigure[Results under 20\% MAR]{ 
        
        \includegraphics[width=0.48\textwidth]{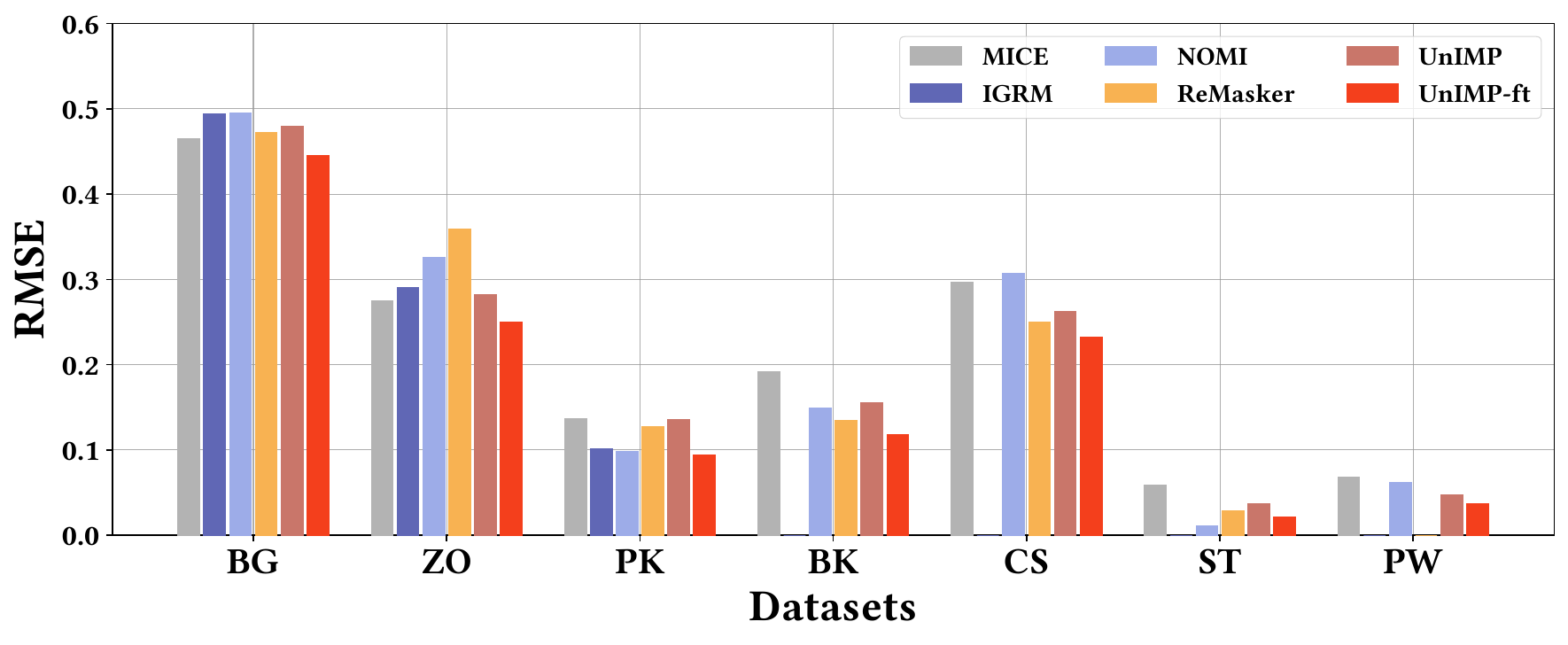}
    }
    \subfigure[Results under 20\% MNAR]{ 
        
        \includegraphics[width=0.48\textwidth]{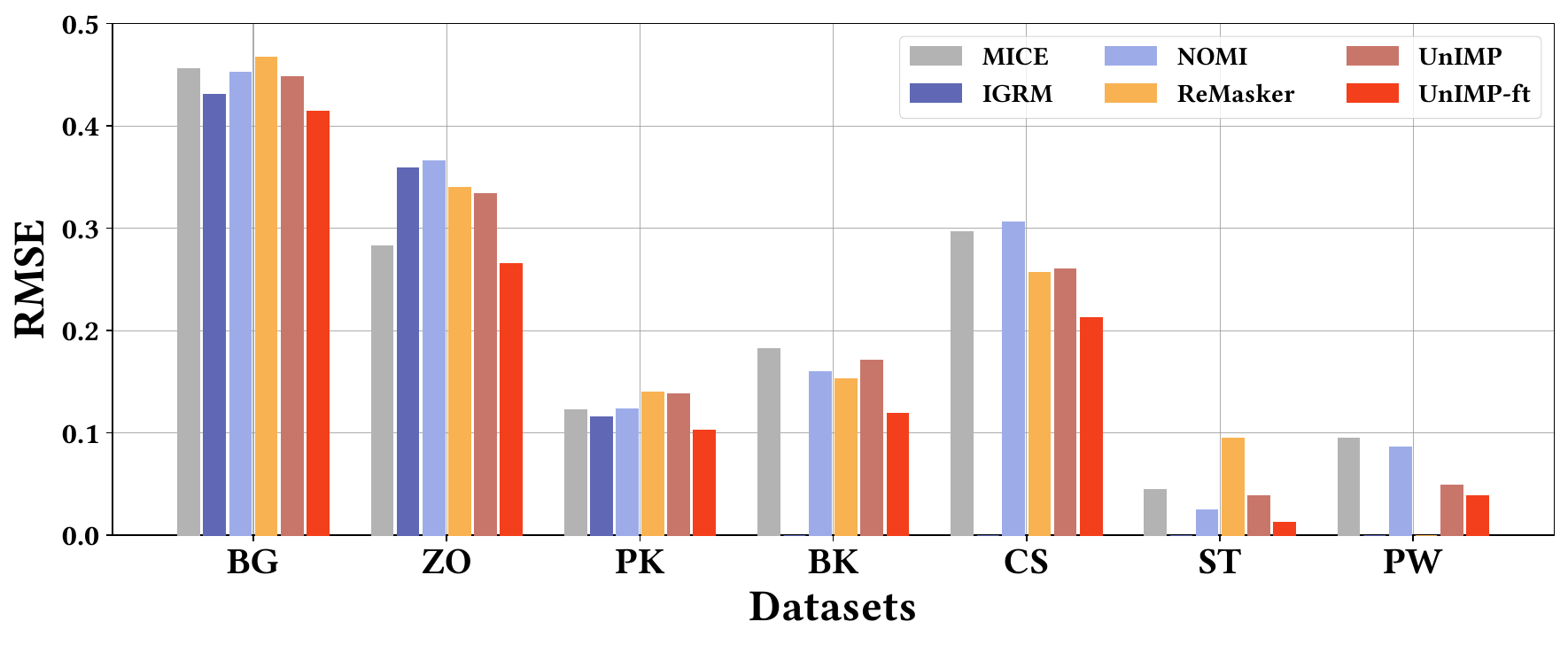}
    }
    \vspace{-2mm}
    \caption{Results of different missing mechanisms}
    \vspace{-4mm}
    \label{fig:mnar_mar}
\end{figure*}

\noindent\textbf{Exp-1: Accuracy over numerical and categorical data}. In this experiment, we evaluate the imputation accuracy for numerical and categorical data. The results are summarized in Table~\ref{tab:MCAR}, with a missing rate of 20\% under the MCAR mechanism. Some results are omitted due to imputation processes exceeding 10 hours (denoted as OOT for "out of time") or running out of memory (OOM).
As shown in the table, LLM-based methods such as DFMs and Table-GPT demonstrate poor performance on both numerical and categorical data, while our method achieves good performance across both data types. Graph-based methods, including GINN, GRAPE, IGRM and \netname~show an outstanding performance. Furthermore, our fine-tuned version, \netname-ft, further enhances performance compared to \netname.
Specifically, in terms of RMSE, \netname-ft shows average performance gains of 12.11\%, 28.27\%, 28.32\% and 29.61\% compared to IGRM, MICE, NOMI and ReMasker, respectively. Regarding MAE, \netname-ft demonstrates average performance gains of 8.74\%, 36.49\%, 30.86\% and 35.36\% over IGRM, MICE, NOMI, and ReMasker, respectively. These results highlight the excellence of \netname~and \netname-ft in imputing numerical and categorical data.

\noindent\textbf{Exp-2: Accuracy over text data}. In this experiment, we evaluate the performance when imputing text data. We compare \netname~with LLM-based methods as other imputation methods are hard to handle text data. The results under 20\% missing rate and MCAR are presented in Table~\ref{tab:string_result}. The results show that our proposed methods, including both \netname~and \netname-ft, consistently outperform previous LLM-based methods. Specifically, regarding $\text{ROUGE-1}_{F_1}$, \netname~has an average performance gain of 30.53\%, 28.04\% and 23.47\% compared to DFMs, Table-GPT and Jellyfish, respectively. Regarding Cos-Sim,  \netname~has an average performance gain of 8.68\%, 5.58\% and 6.88\% compared to DFMs, Table-GPT and Jellyfish, respectively. 
\netname-ft can further boost the performance compared to \netname.
These results collectively highlight the high performance of \netname~and \netname-ft.

\noindent\textbf{Exp-3: Performance under MAR and MNAR}. The missing patterns in the MNAR and MAR are based on the data matrix and mask matrix as shown in ~\cite{miao2022experimental,wang2024missing}. For the text data, it is hard to simulate such a relationship, therefore, we focus on MNAR and MAR for numerical and categorical data in this paper. Four baselines are used for their high performance under MCAR, including the MICE, IGRM, NOMI and ReMasker. 
We follow the simulation process in ~\cite{wang2024missing, yoon2018gain,miao2022experimental} to model the missing mechanism. We use a logistic model built based on randomly selected features to capture the dependency. Missing probability is generated by the logistic model. Consistent with Exp-1 and Exp-2, the missing rate is 20\%.

The overall results under MAR and MNAR are illustrated in Figure~\ref{fig:mnar_mar}. The results show that our methods have excellent performance under both MAR and MNAR. Specifically, under MAR, \netname-ft~reduces the RMSE by 27.58\%, 10.22\%, 4.39\% and 17.86\% compared to MICE, IGRM, NOMI and ReMasker respectively. Under MNAR, \netname-ft reduces the RMSE by 30.23\%, 10.72\%, 27.32\% and 29.57\%, respectively. These results highlight the superiority of \netname~and \netname-ft under MAR and MNAR.

\begin{figure}[t!]
    \centering
    
\setlength{\abovecaptionskip}{0.cm}
    \includegraphics[width=0.48\textwidth]{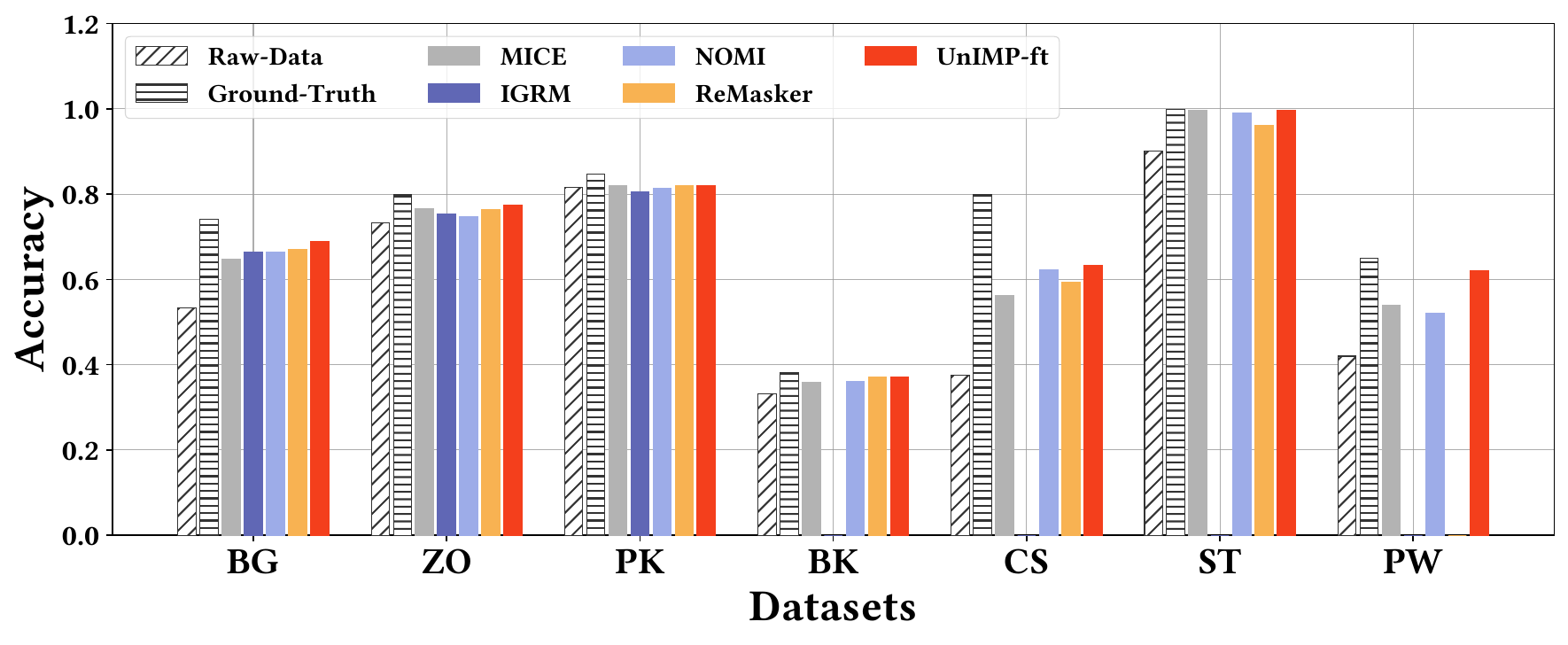} 
    \vspace{-5mm}
    \caption{Results of downstream classification}
    \label{fig:downstream}
\vspace{-5mm}
\end{figure}

\noindent\textbf{Exp-4: Downstream classification accuracy}. In this experiment, we evaluate the downstream classification performance on datasets imputed by MICE, IGRM, NOMI, ReMasker and \netname-ft. For a comprehensive comparison, we also include results from the ground-truth dataset (with no missing data) and the raw dataset (where missing values are replaced with 0, following~\cite{wang2024missing}). This experiment aims to demonstrate that improved dataset quality leads to better downstream classification accuracy.
We employ Random Forest as our classifier for all imputed datasets. The datasets are split into 80\% training and 20\% testing sets. Figure~\ref{fig:downstream} presents the classification results. The results indicate a higher imputation accuracy generally leads to a higher downstream classification, and datasets imputed by our proposed method achieve superior classification results compared to datasets imputed by other baselines.

\vspace{-2mm}
\subsection{Efficiency Evaluation}

\noindent\textbf{Exp-5: Efficiency}.
In this experiment, we evaluate the imputation efficiency. The results are summarized in Figure~\ref{fig:imputation_efficiency}. As our method is related to the graph-based imputation and LLM-based imputation, we mainly involve comparing \netname-ft with graph-based methods (i.e., GRAPE and IGRM) and LLM-based methods (i.e., DFMs, Table-GPT and Jellyfish) and baselines with high imputation accuracy (i.e., NOMI and ReMasker). 
As shown in the figure, our method demonstrates competitive efficiency compared to other baselines. Compared to other graph-based methods, \netname-ft effectively handles large datasets and achieves higher accuracy, with a 3.07$\times$ and 5.38$\times$ speedup over GRAPE and IGRM, respectively. Additionally, compared to other LLM-based methods, \netname-ft benefits from using BiHMP as an adapter and requires fewer prompt tokens for inference, leading to improved imputation efficiency. Collectively, these results highlight the superior efficiency of \netname.

\begin{figure}[t!]
    \centering
    
\setlength{\abovecaptionskip}{0.cm}
    \includegraphics[width=0.48\textwidth]{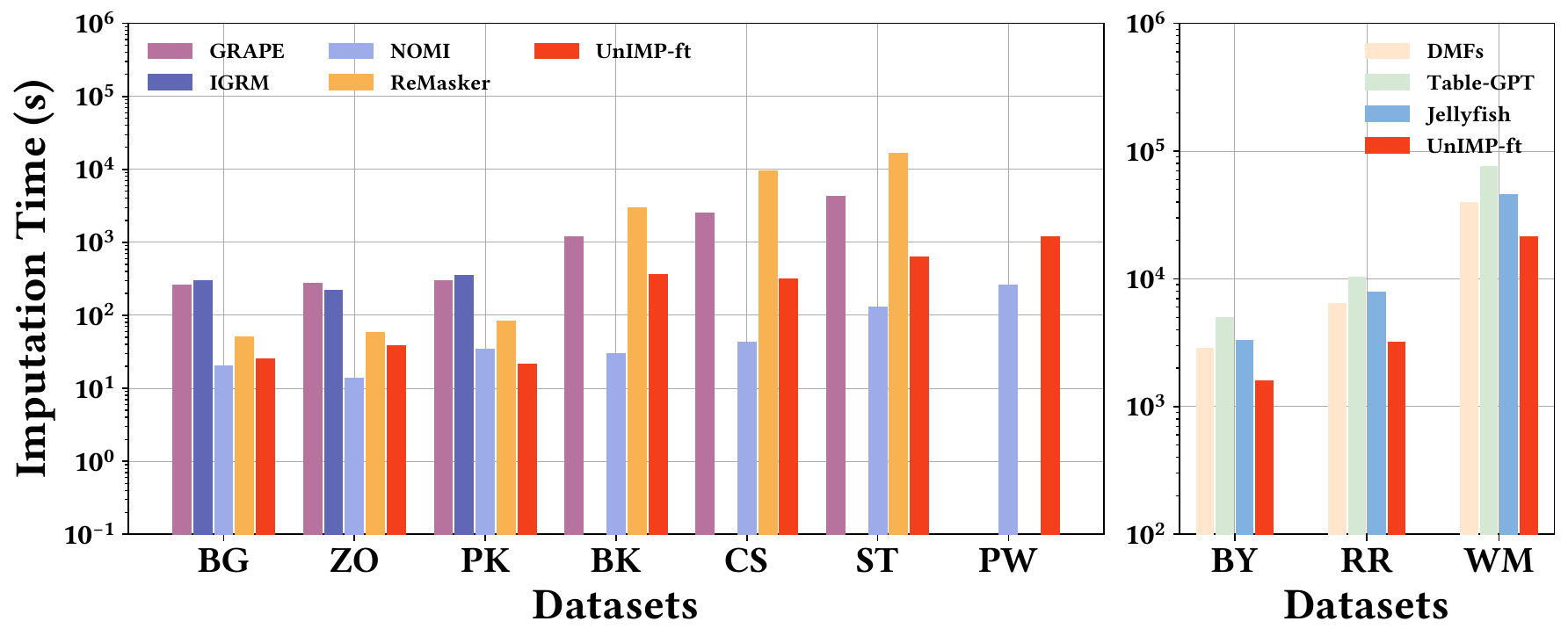} 
    \caption{Imputation efficiency evaluation (in seconds). Besides, the numerical/categorical data and text data need 20824 seconds and 73221 seconds for pre-training \netname, respectively. 45381 seconds are needed for pre-training Table-GPT.}
    \label{fig:imputation_efficiency}
\vspace{-4mm}
\end{figure}

\begin{figure*}
\subfigbottomskip=-8pt 
\subfigcapskip=-7pt 
    \subfigure[Varying LLMs]{ 
        \includegraphics[width=0.132\textwidth]{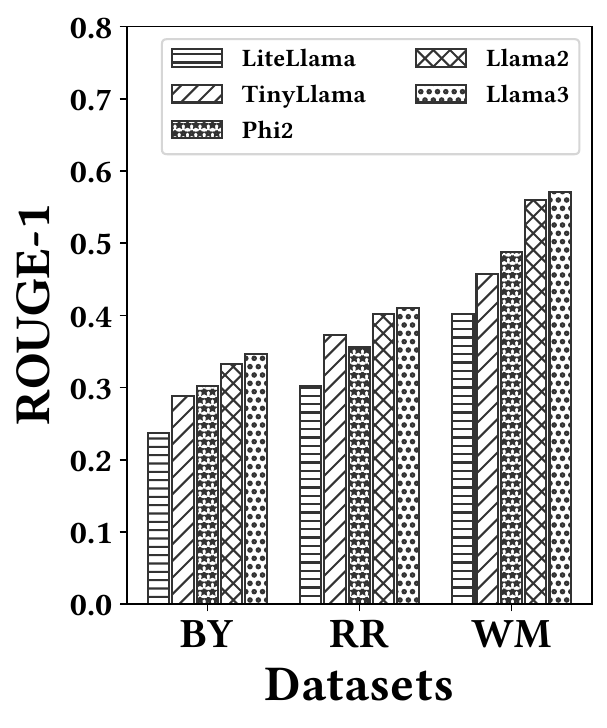}
    }
    \subfigure[Varying chunck sizes]{
        \includegraphics[width=0.132\textwidth]{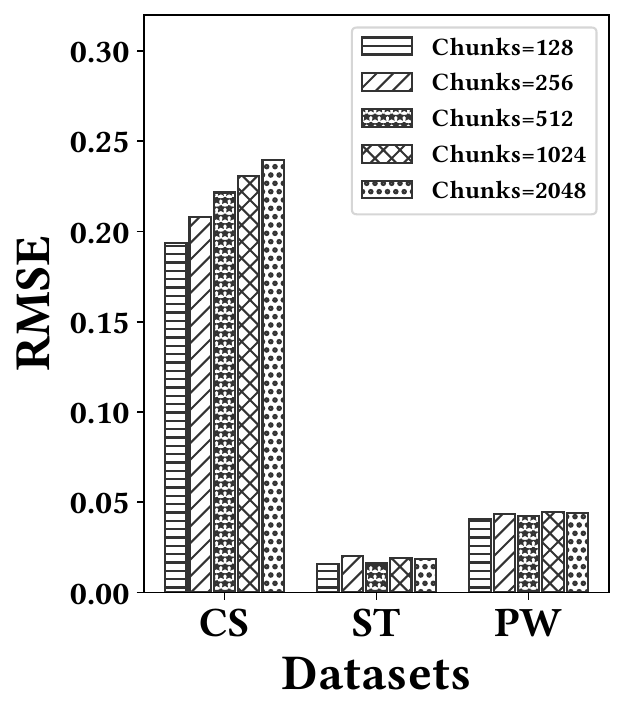} 
        \includegraphics[width=0.132\textwidth]{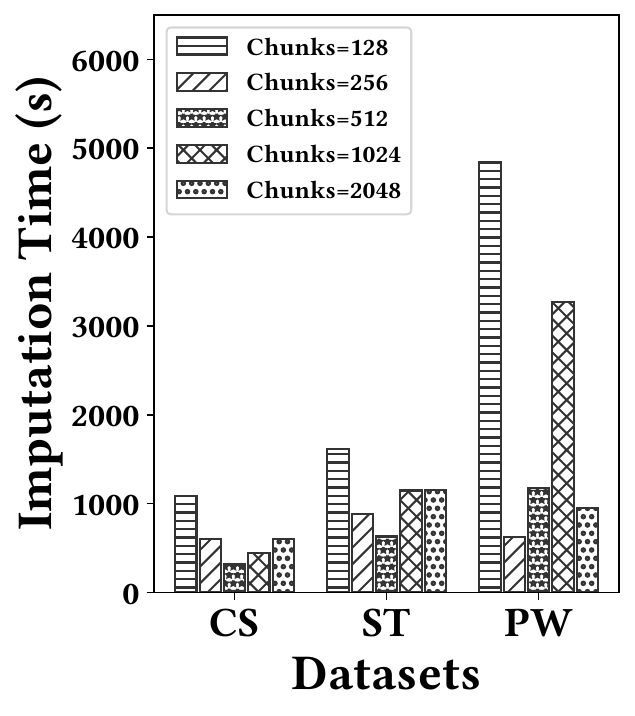} 
    }
    \subfigure[Varying batch sizes]{ 
        \includegraphics[width=0.132\textwidth]{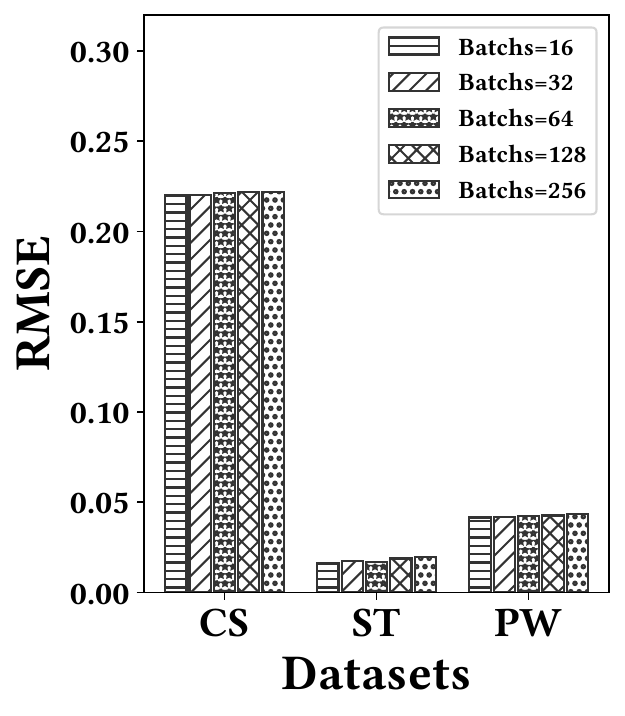}
        \includegraphics[width=0.132\textwidth]{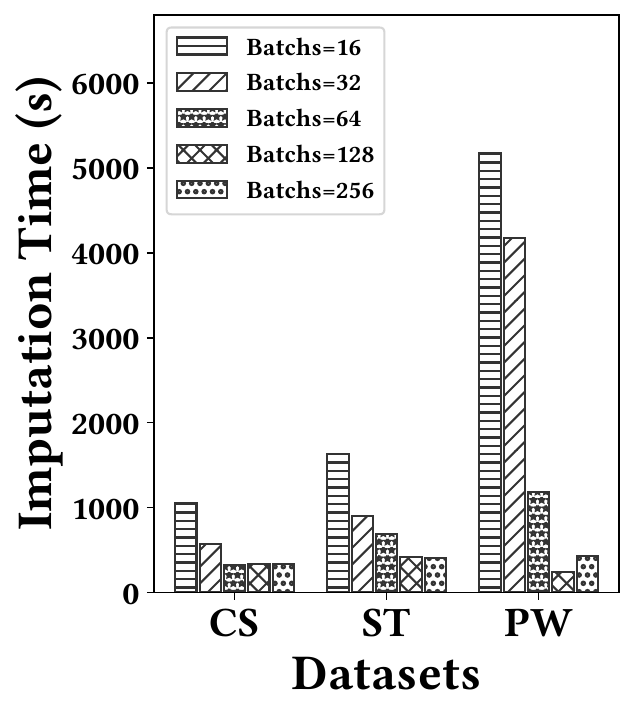}
    }
    \subfigure[Varying $\delta$]{ 
        \includegraphics[width=0.132\textwidth]{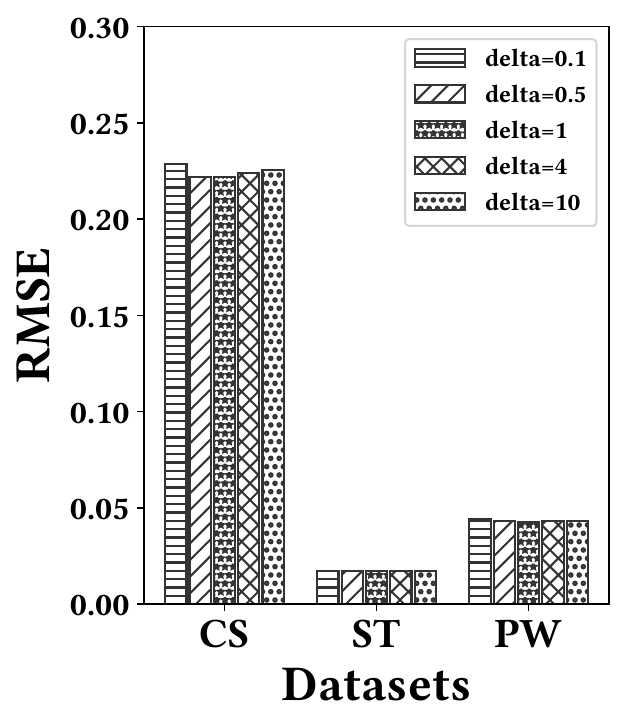}
    }
    \subfigure[Varying $\kappa$]{ 
        \includegraphics[width=0.132\textwidth]{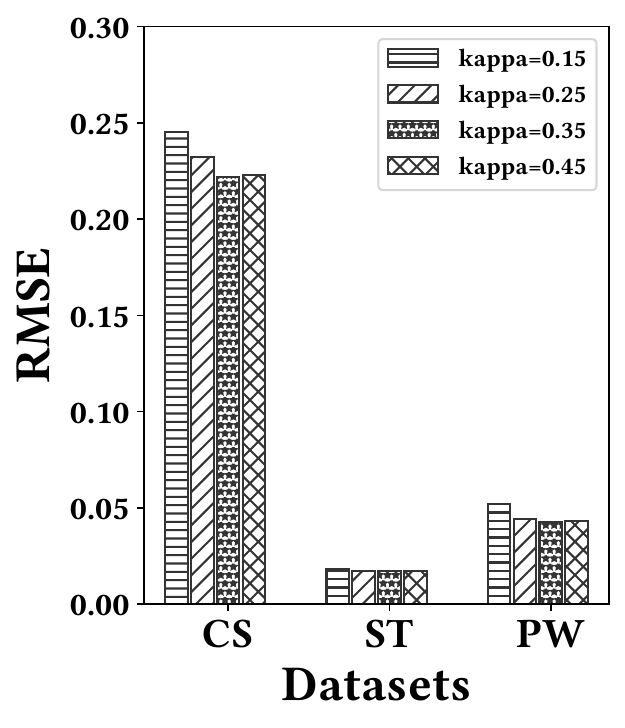}
    }
    \subfigure[F1-score with varying missing rates]{ 
        
        \includegraphics[width=0.32\textwidth]{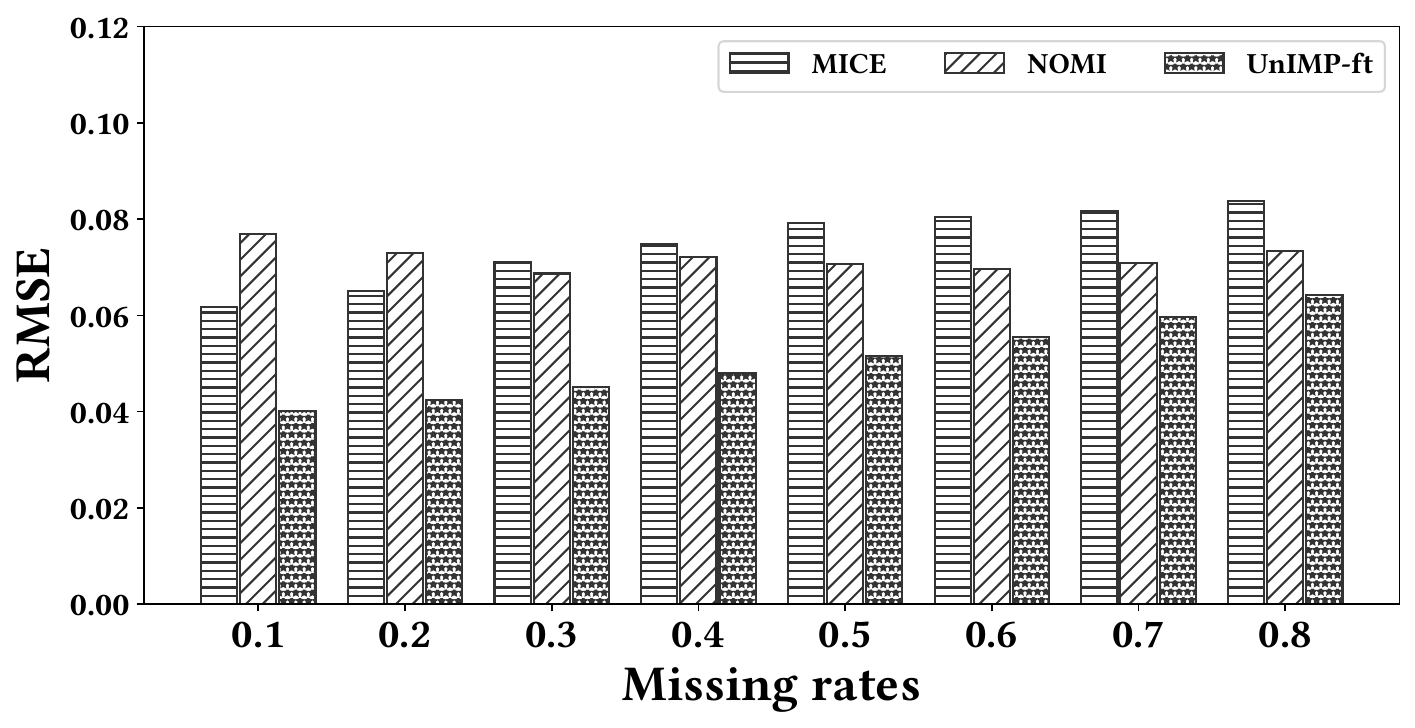}
    }
    \subfigure[RMSE with varying pre-training epochs]{ 
        
        \includegraphics[width=0.32\textwidth]{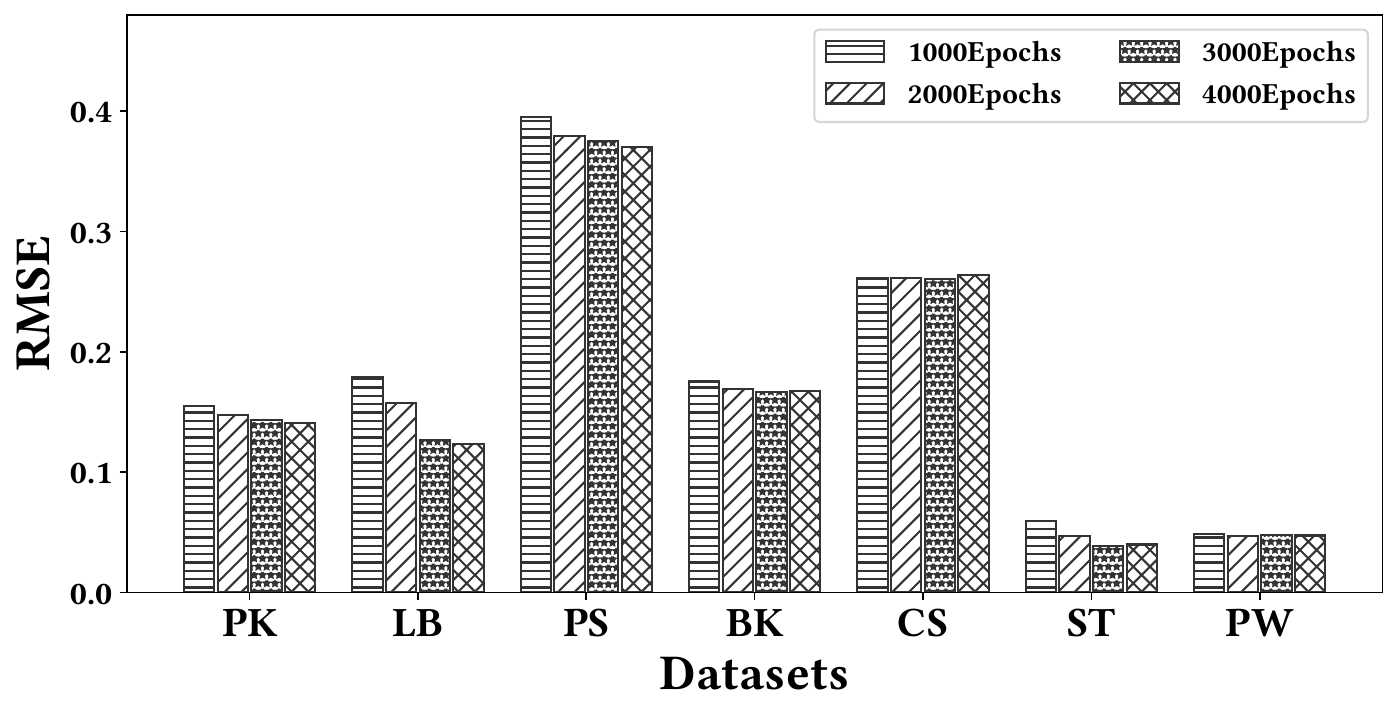}
    }
    \subfigure[RMSE with varying fine-tuning epochs]{
        \includegraphics[width=0.32\textwidth]{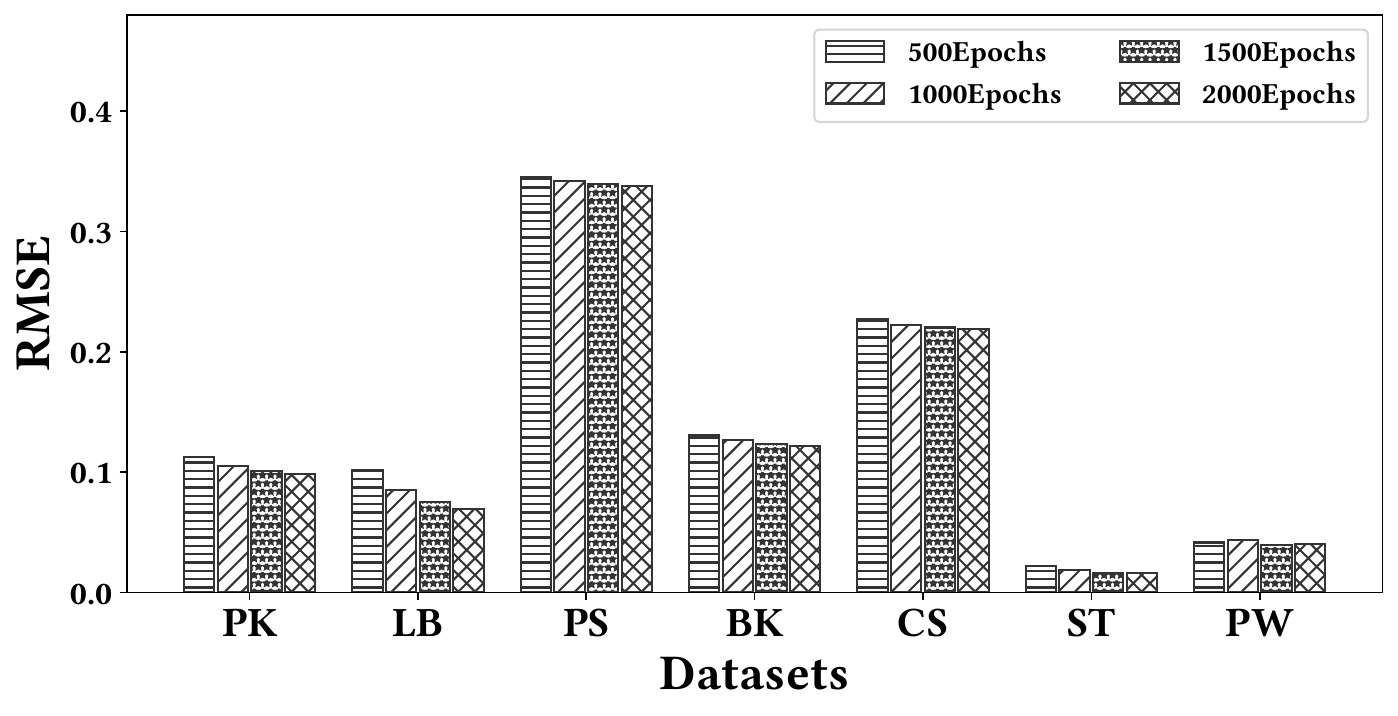} 
    }
    
    \vspace{-2mm}
    \caption{Results of hyper-parameter analysis}
    \vspace{-5mm}
    \label{fig:hyperparameter_analysis}
\end{figure*}

\vspace{-2mm}
\subsection{Ablation Study}

\begin{table}[t] \centering
    
    \vspace{-1mm}
     \caption{Results of ablation study}
     \vspace{-2mm}
    \label{tab:ablation_result}
    \resizebox{0.45\textwidth}{!}{
    \Huge\setlength{\tabcolsep}{4mm}
        \begin{tabular}{c*{3}{c}*{3}{c}}
        \toprule
        & \multicolumn{3}{c}{\textbf{$\text{ROUGE-1}_{F_1}$}} 
        & \multicolumn{3}{c}{\textbf{Cos-Sim}} \\
        \cmidrule(lr){2-4}\cmidrule(lr){5-7}
                               Model & {Buy} & {Restaurant} & {Walmart} & {Buy} & {Restaurant} & {Walmart} \\
        \midrule \midrule
\netname  & 0.3327 &	0.4017 &	0.5594 & 0.8610 &	0.8774 &	0.9025 \\
 w/o Hypergraph    &  0.2580 &	0.3757 &	0.4812 &	0.8508 &	0.8611	& 0.8681\\
 w/o BiHMP    &  0.3014 &	0.3877 &	0.5282 &	0.8571 &	0.8702	& 0.8822\\
 w/o Xfusion  &  0.2867	& 0.3648 &	0.5052 &	0.8517 &	0.8581 &	0.8729   \\
 w/o P.g. mask &  0.3263 &	0.3908 &	0.5337 &	0.8602 &	0.8654 &	0.8932   \\

        \bottomrule
    \end{tabular}
 }
\vspace{-6mm}
\end{table}
\noindent\textbf{Exp-6: Ablation study}. In this section, we conduct an ablation study to evaluate the effectiveness of the key components proposed in \netname. Specifically, we evaluate the hypergraph-aggregated information, the BiHMP module, the Xfusion model and the progressive masking technique.
To evaluate the impact of hypergraph information, we replace it with a zero tensor while keeping all other components the same and train the model. To assess the effectiveness of BiHMP, we substitute it with the Set Transformer from~\cite{chen2024hytrel}, leaving the rest unchanged. For Xfusion, we replace it with a weighted sum fusion mechanism, which combines token and BiHMP-aggregated embeddings by generating weights from the BiHMP-aggregated embeddings and applying them to the token embeddings. Lastly, to analyze the effectiveness of progressive masking, we train the model using all observed data without masking.
The results are summarized in Table~\ref{tab:ablation_result}. As shown, incorporating hypergraph information provides an average improvement of 5.96\% in $\text{ROUGE-1}_{F_1}$ and 2.03\% in $\text{Cos-Sim}$. The BiHMP module contributes an average gain of 2.55\% in $\text{ROUGE-1}_{F_1}$ and 1.04\% in $\text{Cos-Sim}$. Xfusion yields an average improvement of 4.57\% in $\text{ROUGE-1}_{F_1}$ and 1.94\% in $\text{Cos-Sim}$. Progressive masking leads to an average gain of 1.43\% in $\text{ROUGE-1}_{F_1}$ and 0.74\% in $\text{Cos-Sim}$.

\noindent\textbf{Exp-7: Generalization}. In this experiment, we assess the imputation performance on datasets that are not used during pre-training, i.e., in an inductive setting. We use six new datasets from the UCI repository~\cite{Dua:2019}, a widely-used open-source collection frequently referenced in previous studies~\cite{miao2022experimental, zhao2023transformed, wang2024missing, you2020handling}. The datasets are Ionosphere, Breast, Spam, News, Connect and Metro. The results, presented in Table~\ref{tab:inductive_result}, demonstrate that despite \netname~not being trained on these datasets, it achieves competitive performance compared to models trained specifically on them. In particular, for large datasets like Connect and Metro, \netname~has an RMSE difference of within ±10\% compared to models trained directly on those datasets.

\begin{table}[t] \centering
    \caption{RMSE results of inductive generalization}
    \vspace{-3mm}
    \label{tab:inductive_result}
    \resizebox{0.45\textwidth}{!}{
    \Huge\setlength{\tabcolsep}{4mm}
        \begin{tabular}{c*{3}{c}*{3}{c}}
        \toprule
                               Model & {Ionosphere} & {Breast} & {Spam} & {News} & {Connect} & {Metro} \\
        \midrule \midrule
MICE      & 0.2487	& 0.2183	& 0.0712 & 0.0237 &	0.2539 & 0.1225\\
IGRM    &  0.2502 &	0.2107 &	0.0571 &	OOM &	OOM &	OOM    \\
NOMI  & 0.2320 &	0.2115 &	0.0611 &	0.0252 &	0.2473 &	0.1485 \\
ReMasker  & 0.2360 &	0.2058 &  0.0595 & 0.0282 &	OOT &	OOT  \\
\netname & 0.2444 &	0.2426 & 0.0626 &	0.0306 &	0.2569 &	0.1308\\
        \bottomrule
    \end{tabular}
 }
\vspace{-6mm}
\end{table}

\vspace{-2mm}
\subsection{Hyper-parameters Analysis}

\noindent\textbf{Exp-8: Varying LLMs}. We evaluate the imputation RMSE using several representative LLMs, including LiteLlama~\cite{LiteLlama}, TinyLlama, Phi2~\cite{phi2}, Llama2~\cite{touvron2023llama} and Llama3~\cite{dubey2024llama}, with the results shown in Figure~\ref{fig:hyperparameter_analysis}(a). As depicted, more advanced LLM generally leads to better imputation performance. We selected Llama2, as it was one of the most outstanding open-source LLMs with extensive resources and support available when conducting experiments.

\noindent\textbf{Exp-9: Chunk sizes}. We evaluate imputation RMSE and imputation time with varying chunk sizes, as shown in Figure~\ref{fig:hyperparameter_analysis}(b). The chunk sizes tested are 128, 256, 512, 1024 and 2048. The evaluation is conducted on the Chess, Shuttle and Power datasets which have the largest number of samples. As shown, RMSE fluctuates with different chunk sizes, while imputation time first decreases and then increases. This may be because very small chunks underutilize the GPU resources, while very large chunks require excessive I/O resources. A chunk size of 512 strikes a good balance between imputation RMSE and time.

\noindent\textbf{Exp-10: Batch sizes}. We evaluate the imputation RMSE and imputation time with varying batch sizes, as shown in Figure~\ref{fig:hyperparameter_analysis}(c). The batch sizes tested are 16, 32, 64, 128 and 256. As in Exp-9, the Chess, Shuttle, and Power datasets are used. The results show that RMSE fluctuates with different batch sizes, while imputation time decreases as batch size increases. A batch size of 64 generally provides a good imputation RMSE and time.


\noindent\textbf{Exp-11: Varying $\delta$}. We evaluate the imputation RMSE with different values of $\delta$, which is used in the loss function as defined in Eqn~\ref{equ:loss_huber}. The values of $\delta$ tested are 0.1, 0.5, 1, 4 and 10, and the results are shown in Figure~\ref{fig:hyperparameter_analysis}(d). As illustrated, RMSE fluctuates with varying $\delta$, with $\delta=1$ generally providing an excellent RMSE.

\noindent\textbf{Exp-12: Varying ratios of progressive masked data, $\kappa$}. We evaluate the imputation RMSE with different values of $\kappa$, which is used to control the newly masked ratio. The values tested are 0.15, 0.25, 0.35 and 0.45, and the results are shown in Figure~\ref{fig:hyperparameter_analysis}(e). As illustrated, RMSE first decreases and then increases with $\kappa=0.35$ generally providing an excellent RMSE.

\noindent\textbf{Exp-13: Missing rates}.
We experiment with missing rates ranging from 0.1 to 0.8 on the Power dataset which contains the largest number of samples. We compare \netname~with MICE and NOMI, as they demonstrated superior performance on Power as in Exp-1. The results are presented in Figure~\ref{fig:hyperparameter_analysis}(f). As illustrated, the RMSE increases with higher missing rates, and \netname-ft consistently outperforms the baselines across varying missing rates.


\noindent\textbf{Exp-14: Pre-training epochs}. We assess the performance across pre-training epochs, varying from 1000 to 4000 in increments of 1000. The results are presented in Figure~\ref{fig:hyperparameter_analysis}(g). As shown, RMSE generally decreases with more pre-training epochs, and 4000 epochs is generally sufficient to achieve superior performance.

\noindent\textbf{Exp-15: Fine-tuning epochs}. We evaluate the performance under varying fine-tuning epochs, ranging from 500 to 2000 at an interval of 500. The results are illustrated in Figure~\ref{fig:hyperparameter_analysis}(h). As shown, RMSE generally decreases with more fine-tuning epochs, and 1000 epochs is generally adequate to achieve excellent performance.

\vspace{-3mm}
\section{Related Work}
\vspace{-1mm}
\label{sec:relatedwork}

\noindent \textbf{Statistics-based and similarity-based imputation.}
Statistics-based algorithms impute missing data using statistical values (e.g., mean or mode) for each feature~\cite{moons2006using,farhangfar2007novel}. Similarity-based algorithms impute missing data using closely related data. KNNI~\cite{altman1992introduction} selects K nearest neighbors of the sample with missing features for imputation. An iterative method proposed in~\cite{zhang2012nearest} locates a more reliable local context. NOMI~\cite{wang2024missing} uses similarity to iteratively augment the input and employs an uncertainty-driven network for imputation.

\noindent \textbf{Graph-based models for imputation.} Graph structures have been employed to capture complex relationships for data imputation. One popular approach involves using similarity graphs, as demonstrated in GINN~\cite{spinelli2020missing} and GEDI~\cite{chen2023gedi}, where edges are constructed between samples to facilitate imputation. Another popular approach leverages bipartite graphs, as seen in GRAPE~\cite{you2020handling} and IGRM~\cite{zhong2023data}, where each row and column is represented as a node, and the cell values are modeled as edges in the graph.

\noindent \textbf{Deep learning techniques for imputation.}
Various AI techniques are used to better capture the underlying relationships in raw datasets. Traditional methods include models like XGBoost~\cite{chen2016xgboost}, random forests in MissForest~\cite{stekhoven2012missforest}, multi-layer perceptrons~\cite{garcia2010pattern}, and linear regression in MICE~\cite{royston2011multiple}. Generative models are also widely employed, with autoencoders (AEs) being particularly popular. Examples include deep denoising autoencoders (DAE) in MIDAE~\cite{gondara2018mida} and ReMasker~\cite{DBLP:conf/iclr/DuM024}, variational autoencoders (VAE)\cite{mccoy2018variational}, and importance-weighted autoencoders in MIWAE\cite{mattei2019miwae}. GANs are another commonly used approach for missing data imputation, as seen in GAIN~\cite{yoon2018gain} and VGAIN~\cite{miao2022experimental}.
Recently, distribution-matching-based imputation methods have been introduced for missing data imputation. These approaches align missing values with observed data distributions in the original space (e.g., OTImputer~\cite{muzellec2020missing}) and in the latent space (e.g., TDM~\cite{zhao2023transformed}).

\noindent \textbf{LLMs-based imputation.}
Recently, LLMs have shown a remarkable ability to understand and generate human-like text across a diverse array of tasks. Two main approaches have emerged: in-context learning and fine-tuning. In the in-context-learning-based approach, LLMs are prompted to predict missing data directly with task-related examples, as demonstrated in~\cite{narayan2022can,qian2024unidm}. Fine-tune-based methods, such as Table-GPT~\cite{li2024table} and Jellyfish~\cite{zhang2024jellyfish}, fine-tune the LLMs to adapt the model to tasks in the field of tabular data.

\vspace{-3mm}
\section{Conclusion}
\vspace{-1mm}
\label{sec:conclusion}

In this paper, we study the mixed-type data imputation and propose \netname, an LLM-enhanced unified imputation framework with high-order message passing. We first introduce a cell-oriented hypergraph to better capture the key properties of tabular structures. We then design BiHMP to learn to aggregate global-local and high-order information while capturing the column patterns through message passing on the hypergraph. 
We propose Xfusion to effectively align BiHMP with the powerful capabilities of LLM. BiHMP and Xfusion serve as the adapter of the LLM. To train \netname, we introduce chunking and progressive masking techniques within a pre-training and fine-tuning strategy. 
Theoretical and empirical studies highlight the superiority of \netname.

\balance
{
\bibliographystyle{ACM-Reference-Format}
\bibliography{sample}
}

\end{document}